\documentclass[accepted]{uai2026} 
                        
\usepackage[american]{babel}

\usepackage{natbib} 
\usepackage{mathtools} 
\usepackage{booktabs} 
\usepackage{tikz} 

\usepackage{amsmath,amsfonts,bm}

\def\eqref#1{equation~\ref{#1}}

\def\1{\bm{1}}

\def\rva{{\mathbf{a}}}

\def\rvd{{\mathbf{d}}}

\def\rvg{{\mathbf{g}}}

\def\rvu{{\mathbf{i}}}

\def\rvm{{\mathbf{m}}}

\def\rvu{{\mathbf{u}}}
\def\rvv{{\mathbf{v}}}

\def\rvx{{\mathbf{x}}}
\def\rvy{{\mathbf{y}}}
\def\rvz{{\mathbf{z}}}

\DeclareMathAlphabet{\mathsfit}{\encodingdefault}{\sfdefault}{m}{sl}
\SetMathAlphabet{\mathsfit}{bold}{\encodingdefault}{\sfdefault}{bx}{n}

\DeclareMathOperator*{\argmin}{arg\,min}

\usepackage{mdframed}
\usepackage{amssymb}
\usepackage[normalem]{ulem}
\usepackage{amsthm}
\usepackage{thmtools} 
\usepackage{algorithm}
\usepackage{algorithmic}
\usepackage{thm-restate}
\theoremstyle{plain}
\newtheorem{theorem}{Theorem}[section]

\newtheorem{lemma}[theorem]{Lemma}
\newtheorem{corollary}[theorem]{Corollary}
\theoremstyle{definition}
\newtheorem{definition}[theorem]{Definition}
\newtheorem{assumption}[theorem]{Assumption}
\theoremstyle{remark}
\newtheorem{remark}[theorem]{Remark}

\title{The Convergence Behavior of Adam under Heavy-Tailed Noise}

\author[1]{\href{mailto:<pangyiji@msu.edu>?Subject=The Convergence Behavior of Adam under Heavy-Tailed Noise}{Yijiang Pang}{}}
\affil[1]{%
    Michigan State University\\
    East Lansing, USA
}  
  
\begin{document}
\maketitle

\begin{abstract}
We establish the first convergence guarantees for the plain vector-form \emph{Adam} optimizer under heavy-tailed stochastic noise. While several Adam variants are known to achieve optimal iteration complexity in bounded-variance nonsmooth nonconvex optimization, little is understood about their behavior when stochastic gradients admit only a bounded $p$-th central moment for some $p \in (1,2]$, a setting increasingly observed in modern deep learning. To address this gap, we generalize the recent online-to-nonconvex conversion framework to accommodate heavy-tailed martingale-difference noise. Building on this generalized framework, we develop a discounted regret analysis for Adam, without restrictive parameter coupling.
Our results show that Adam converges to $(\rho,\epsilon)$-stationary points under heavy-tailed noise. However, it exhibits a suboptimal iteration complexity and $p$-dependent convergence, a suboptimality that persists even in the bounded-variance case ($p=2$). Specifically, the $\epsilon$-dominant term in the iteration complexity for reaching in-expectation stationarity is
\(
T
=
\mathcal{O}\left(
\Delta \rho^{1/2}
(G+\sigma)^{\frac{5p}{3p-4}}
\epsilon^{-\left(\frac{5p}{3p-4}+\frac{3}{2}\right)}
\right)
\) for $p \in (\frac{4}{3},2]$,
which simplifies to $T=\mathcal{O}(\epsilon^{-13/2})$ when $p=2$. When the domain radius is known and used to control the online-learner output, a standard setup in related literature, the convergence rate improves to match the optimal complexity. In this case, the $\epsilon$-dominant iteration complexity is
\(
T
=
\mathcal{O}\left(
\Delta \rho^{1/2}
(G+\sigma)^{\frac{p}{p-1}}
\epsilon^{-\left(\frac{p}{p-1}+\frac{3}{2}\right)}
\right)
\) for $p \in (1,2]$,
which simplifies to $T=\mathcal{O}(\epsilon^{-7/2})$ when $p=2$. These findings provide new theoretical insight into the robustness and limitations of Adam in heavy-tailed regimes.
\end{abstract}

\section{Introduction}
Adaptive gradient methods, most notably Adam~\citep{kingma2014adam}, have become the de facto optimization tools for training modern deep neural networks.
Despite their remarkable empirical success, the theoretical understanding of Adam-type algorithms remains incomplete, particularly in stochastic environments where gradient noise deviates from classical bounded-variance assumptions.

Recent empirical and theoretical studies~\citep{battash2024revisiting, ahn2023linear, garg2021proximal} suggest that stochastic gradients in deep learning often exhibit heavy-tailed behavior. Instead of satisfying a bounded-variance condition, the noise may only admit a bounded $p$-th central moment for some $p \in (1,2]$. Such heavy-tailed phenomena have been observed in large-scale language model training and other high-dimensional learning problems. In these regimes, analyses relying on sub-Gaussian or bounded-variance assumptions are no longer applicable, and even classical stochastic gradient descent (SGD)~\citep{robbins1951stochastic} may suffer from instability or divergence~\citep{zhang2020adaptive}.

Over the past few years, substantial progress has been made in understanding both non-adaptive and adaptive methods under heavy-tailed noise. Clipped SGD, normalized gradient methods, and sign-based optimizers have been shown to provide improved robustness in heavy-tailed settings~\citep{zhang2020adaptive, liu2023breaking, sun2024gradient, he2025complexity}.
On the adaptive side, recent works analyze clipped or modified variants of AdaGrad and Adam-type methods and establish (high-probability) convergence guarantees~\citep{chezhegov2024clipping, fradin2025tight}. More recently, classical algorithms have been shown to achieve optimal convergence guarantees in convex settings and near-optimal rates for non-smooth non-convex optimization under heavy-tailed noise~\citep{liu2025online}.

Nevertheless, convergence guarantees for the Adam update under heavy-tailed noise, without algorithmic modification, remain largely unexplored. In fact, until very recently, rigorous convergence results for Adam-type methods under heavy-tailed stochastic gradients were unavailable~\citep{yu2026sign}. Even in the bounded-variance case ($p=2$), many existing analyses focus on simplified or modified variants of Adam rather than the exact algorithm used in practice.

A major conceptual development in recent years is the online-to-nonconvex conversion framework~\cite{cutkosky2023optimal, ahn2024understanding, zhang2024random}, which establishes a principled connection between adaptive optimization and online learning. We briefly recall the fundamental setup.

Online Linear Optimization (OLO) proceeds over $T$ rounds. At round $t$, the learner selects $\rvz_t \in \mathbb{R}^d$ and then observes a linear loss
$
\ell_t(\cdot) := \langle \rvv_t, \cdot \rangle,
$
where $\rvv_t$ is a cost vector revealed by the environment.
The goal is to minimize the \emph{static regret}
$
\mathrm{Regret}_T(\rvu)
:=
\sum_{t=1}^T \langle \rvv_t, \rvz_t - \rvu \rangle,
$
for any comparator $\rvu$ in a domain $\mathcal{D}$.
To connect online learning with non-convex optimization, recent works introduce the notion of \emph{$\beta$-discounted regret}.
The discounted loss is defined as
$
\ell_t^{[\beta]}(\cdot) = \langle \beta^{-t} \rvv_t, \cdot \rangle,
$
leading to the discounted regret
\[
\mathrm{Regret}_T^{[\beta]}(\rvu)
=
\sum_{t=1}^T
\langle \beta^{T-t}\rvv_t, \rvz_t - \rvu \rangle.
\]

Within this framework, effective nonconvex optimizers correspond to online learners achieving small discounted regret. This perspective provides structural insight into adaptive methods. For example, interpreting Adam as a Follow-the-Regularized-Leader (FTRL) procedure clarifies its update structure and helps explain its empirical advantages over SGD when problem properties are unknown~\citep{ahn2024adam}.
Moreover, the conversion framework has enabled sharp iteration complexity results for Adam. However, existing analyses often rely on technical simplifications, for example, removing bias-correction terms $1/(1-\beta_1^t)$ and $1/(1-\beta_2^t)$ or imposing parameter couplings such as $\beta_2 = \beta_1^2$ for analytical tractability. While these assumptions facilitate proof techniques, they depart from the exact algorithm used in practice and partially obscure its intrinsic behavior.
Furthermore, the current conversion framework does not directly accommodate heavy-tailed noise.

Given these developments, a central question remains open:

\emph{What convergence guarantees can be established for the Adam update, without restrictive modifications? This question becomes even more compelling in heavy-tailed settings.}

\paragraph{Our contributions.}
In this work, we provide the first convergence guarantees for the Adam update under heavy-tailed stochastic noise. Our main contributions are summarized as follows:

\begin{itemize}
    \item \textbf{Heavy-tail conversion theory.} We generalize the online-to-nonconvex conversion framework to accommodate heavy-tailed martingale-difference noise.
    
    \item \textbf{Vector-form Adam regret analysis.}
    We establish a complete discounted regret analysis for the exact vector-form Adam update without enforcing restrictive parameter coupling: (a) we only require $\beta_{1}\le \beta_{2}$; (b) the numerical stabilizer $\epsilon$ in
    Adam's denominator is treated as a fixed small constant and is not required to scale with either $G$ or $\sigma$.
    
    \item \textbf{Heavy-tailed nonconvex convergence.} We show that Adam converges to $(\rho,\epsilon)$-stationary points under heavy-tailed noise. The resulting complexity is
    $p$-dependent and suboptimal relative to the optimal nonsmooth stochastic rate; this suboptimality remains even in the bounded-variance case $p=2$.
    
    \item \textbf{Optimal rates under known domain radius.} When the domain radius is known and used to clip the online-learner output, a standard assumption in related literature, the iteration complexity improves to match optimal rates.
\end{itemize}

\section{Related Work}

\paragraph{Heavy-tailed stochastic optimization}
Heavy-tailed gradient noise, typically modeled by assuming only a bounded $p$-th central moment for some $p\in(1,2]$, has been widely studied as a realistic alternative to bounded-variance assumptions in deep learning.
In non-adaptive optimization, a large body of work develops robust methods for this regime.
\citet{zhang2019gradient} analyzed clipped SGD and established convergence guarantees together with lower bounds under heavy-tailed noise.
\citet{gorbunov2020stochastic} and \citet{nguyen2023improved} strengthened these results to high-probability guarantees and studied accelerated variants.
Beyond explicit clipping, \citet{gorbunov2020stochastic} and \citet{yu2026sign} analyzed normalization-based and sign-based methods, showing that magnitude suppression can achieve optimal or near-optimal rates without explicit truncation.
For adaptive methods, theoretical guarantees are comparatively limited.
\citet{chezhegov2024clipping} studied clipped AdaGrad- and Adam-type algorithms and established high-probability convergence guarantees under heavy-tailed noise.
\citet{ilboudo2023adaterm} proposed robust moment estimators based on heavy-tailed modeling and analyzed their stability.
These works demonstrate that adaptive methods can be made robust in heavy-tailed regimes, but require explicit algorithmic modifications.

\paragraph{Online-to-nonconvex conversion and heavy-tailed online learning}
A major recent development in adaptive optimization is the online-to-nonconvex conversion framework.
\citet{cutkosky2023optimal, ahn2024understanding, zhang2024random, ahn2024general} connected nonconvex stationarity guarantees to regret bounds of online learners through discounted regret.
These works established optimal iteration complexity in the bounded-variance setting for carefully designed online learners with momentum-type updates.
Parallel to these developments, heavy-tailed online learning has also been studied.
\citet{zhang2022parameter} established high-probability regret guarantees via gradient truncation techniques for online convex optimization.
\citet{liu2025online} established regret guarantees for classical online convex optimization algorithms without explicit truncation.

\section{Preliminaries}

In this section, we introduce the assumptions and basic notions used throughout the paper.
Our assumptions largely follow prior work on non-smooth non-convex stochastic optimization and adaptive methods
\citep{cutkosky2023optimal, ahn2024adam, zhang2024random},
while explicitly allowing only bounded $p$-th moments of the stochastic gradient noise.
We also recall the notion of $(\rho,\epsilon)$-stationarity, which serves as the optimality criterion in our analysis.
This notion is standard in recent studies of non-smooth non-convex stochastic optimization
\citep{zhang2024random, ahn2024adam, zhang2019gradient, jordan2023deterministic, tian2022finite}.
Although $(\rho,\epsilon)$-stationarity relaxes Goldstein stationarity~\citep{goldstein1977optimization},
it retains sufficient structure to recover first-order stationarity under additional smoothness assumptions.

\begin{assumption}
\label{ass_func}
Let $F:\mathbb{R}^{d}\rightarrow\mathbb{R}$ be a differentiable function with the following properties:
\begin{itemize}
    \item  The \textbf{function} is bounded below: $\inf_{\rvx}F(\rvx)>-\infty$. We define $\Delta:=F(\rvx_0)-\inf_{\rvx}F(\rvx)$.
    \item The \textbf{function} F is well-behaved, i.e., $\forall \rvx$ and $\rvy$, $F(\rvy) - F(\rvx) = \int_{0}^{1}\langle\nabla F(\rvx + t(\rvy - \rvx), \rvy - \rvx)\rangle dt$.
    \item The \textbf{function} F is $G$-Lipschitz, i.e., $\forall \rvx, \Vert\nabla F(\rvx)\Vert\leq G$.
    \item For every query point $\rvx$, the \textbf{stochastic gradient}
$\rvg\leftarrow \mathrm{StoGrad}(\rvx,r)$ satisfies
\(
\mathbb{E}[\rvg\mid \rvx]=\nabla F(\rvx),
\quad
\mathbb{E}\!\left[\|\rvg-\nabla F(\rvx)\|^p\mid \rvx\right]\le \sigma^p.
\)
\end{itemize}
\end{assumption}

\begin{definition}[$(\rho, \epsilon)$-stationary point]
\label{def_stat_point}
Supposing $F:\mathbb{R}^{d}\rightarrow\mathbb{R}$ is differentiable. Then $\rvx$ is a $(\rho, \epsilon)$-stationary point of $F$ if $\Vert\nabla F(\rvx)\Vert^{[\rho]} \leq \epsilon$ where
$
\Vert\nabla F(\rvx)\Vert^{[\rho]}:=\inf_{p\in\mathcal{P}(\mathbb{R}^{d}), \mathbb{E}_{\rvy\sim p}[\rvy]=\rvx}\left\{\Vert\mathbb{E}[\nabla F(\rvy)]\Vert + \rho\mathbb{E}[\Vert\rvy-\rvx\Vert^{2}]\right\}.
$
\end{definition}

To facilitate an online-learning interpretation of Adam,
we introduce an equivalent assumption on the cost vectors corresponding to the stochastic gradients in Assumption~\ref{ass_func}.

\begin{assumption}
\label{ass_vector_assumption}
For each $t\in[T]$, let $\rvv_t\in\mathbb{R}^d$ denote a stochastic cost vector
revealed after the history $\mathcal{F}_{t-1}$. Define
\[
\boldsymbol{\mu}_t:=\mathbb{E}[\rvv_t\mid\mathcal{F}_{t-1}],
\qquad
\boldsymbol{\xi}_t:=\rvv_t-\boldsymbol{\mu}_t .
\]
Assume that, almost surely,
\[
\|\boldsymbol{\mu}_t\|\le G,
\quad
\mathbb{E}\!\left[\|\boldsymbol{\xi}_t\|^p\mid\mathcal{F}_{t-1}\right]\le \sigma^p,
\quad
\mathbb{E}\!\left[\boldsymbol{\xi}_t\mid\mathcal{F}_{t-1}\right]=0 .
\]
\end{assumption}




\begin{assumption}[Ball-constrained domain]
\label{ass_bounded_domain}
Let $\mathcal{D}\subseteq\mathbb{R}^d$ denote the Euclidean ball of
radius $D>0$,
$
\mathcal{D}
:=\{x\in\mathbb{R}^d:\|x\|\le D\}.
$
Any comparator $\rvu\in\mathcal{D}$ satisfies $\|\rvu\|\le D$.
\end{assumption}
Assumption~\ref{ass_bounded_domain} is used to specify the comparator class in the online-learning analysis.

\begin{remark}
The symbol $D$ plays two related but distinct roles. For the Adam regret bound, $D$ is only the comparator radius and need not be used by the algorithm. In the known-domain-radius setting~\citep{ahn2024adam, liu2025online}, $D$ is additionally available to the algorithm and can be used to control the online-learner output. For example, \citet{ahn2024adam} employ a $D$-based clipping operation $
\mathrm{clip}_D(\rvz)
:=
\frac{\rvz}{\|\rvz\|}\min(D, \|\rvz\|),
$ to control the outputs of the online learner. 
In our analysis, this distinction explains the difference between the plain-Adam rate and the optimal rate obtained when $D$-based output control is available.
\end{remark}
\section{Discounted-to-nonconvex conversion}
\label{sec_dtnc}

\begin{algorithm}[t]
   \caption{Discounted-to-nonconvex conversion algorithm~\citep{zhang2024random}}
   \label{alg_dtnc}
\begin{algorithmic}[1]
   \STATE {\bfseries Input:} Initial point $\rvx_{0}$, $T$,  online learner $\mathcal{A}$ outputting $\rvz$, and discounting factor $\beta\in(0,1)$
   \FOR{$t=1$ {\bfseries to} $T$}
   \STATE Receive $\rvz_{t}$ from $\mathcal{A}$ 
   \STATE Update $\rvx_{t}\leftarrow \rvx_{t-1} + \kappa_{t}\rvz_{t}$, where $\kappa_{t}\sim\text{Exp(1)}$ i.i.d.
   \STATE Compute $\rvg_{t}\leftarrow \text{StoGrad}(\rvx_{t},r_{t})$
   \STATE Send $\ell_{t}^{[\beta]}(\rvz_{t}):=\langle\beta^{-t}\rvg_{t}, \rvz_{t}\rangle$ to $\mathcal{A}$
   \ENDFOR
\end{algorithmic}
\end{algorithm}

We now present a generalized discounted-to-nonconvex conversion framework that accommodates both heavy-tailed and bounded-variance noise. The results of this section are of independent interest beyond the analysis of Adam.

To obtain a $(\rho,\epsilon)$-stationary point, both components in Definition~\ref{def_stat_point} must be controlled:
(i) the expected gradient norm/smoothed first-order term $\|\mathbb{E}[\nabla F(\rvy)]\|$, and
(ii) the variance term $\mathbb{E}[\|\rvy - \rvx\|^2]$.
We bound these two terms separately.

\subsection{Bounding the Two Terms}

\paragraph{Bounding the expected gradient norm via discounted regret}
We first establish a heavy-tailed analogue of the classical interaction between expected gradient norm and discounted regret.

The proof follows the standard discounted-to-nonconvex conversion argument~\citep{ahn2024adam}, with a key modification to accommodate heavy-tailed noise. In particular, we invoke the von Bahr--Esseen inequality for martingale differences~\citep{von1965inequalities} to control stochastic error terms without assuming bounded variance.
The key result is summarized in Proposition\ref{prop_heavy_tail_conversion}, together with its supporting Lemma~\ref{lemma_discounted_sum_iid}; the proofs are deferred to Appendix~\ref{append_sec_dtnc}.

\begin{restatable}[]{lemma}{LemmaDiscountedSumIID}
\label{lemma_discounted_sum_iid}
Suppose Assumption~\ref{ass_vector_assumption} holds. Then for any $\beta \in (0,1)$, $n \in [T]$, and comparator $\rvu_n \in \mathcal{D}$,
\[
\mathbb{E}\sum_{t=1}^{n}\beta^{n-t}\langle\boldsymbol{\xi}_{t}, \rvu_{n}\rangle \le 2(1 - \beta)^{-\frac{1}{p}}\sigma D.
\]
\end{restatable}

\begin{restatable}[]{proposition}{LemmaHeavyTailConversion}
\label{prop_heavy_tail_conversion}
Suppose $F$ satisfies Assumption~\ref{ass_func}.
For any comparator sequence $\{\rvu_t\} \subseteq \mathcal{D}$, Algorithm~\ref{alg_dtnc}
\begin{align*}
&\mathbb{E}_{\tau}\left\|\mathbb{E}_{\rvy_{\tau}}\nabla F(\rvy_{\tau})\right\| \le 2(1-\beta + \frac{\beta}{T})(1 - \beta)^{-\frac{1}{p}}\sigma + \frac{\Delta}{DT}\nonumber\\
& +\frac{1}{DT}\left((1-\beta)\sum_{t=1}^{T}\mathbb{E}\mathrm{Regret}^{[\beta]}_t(\rvu_t) + \beta\mathbb{E}\mathrm{Regret}^{[\beta]}_T(\rvu_T)\right)
\end{align*}
where $\tau$ is a random index in $[T]$ with distribution $\mathbb{P}(\tau = t) = \begin{cases}
\frac{1-\beta^{t}}{T} & \text{if}\quad t = 1,\cdots,T-1,\\
\frac{1 - \beta^{T}}{(1-\beta)T} & \text{if} \quad t = T
\end{cases}$.
Moreover, $\rvy_t$ is sampled from $\{\rvx_s\}_{s=1}^t$ according to $\mathbb{P}(\rvy_t=\rvx_s) \;=\; \frac{(1-\beta)\beta^{t-s}}{1-\beta^t}, s=1,\ldots,t$.
\end{restatable}

When $p=2$, the above result recovers the existing discounted-to-nonconvex conversion under bounded variance (Lemma 7 of~\citet{ahn2024adam}).

\paragraph{Bounding the variance term}

We now bound the second term in Definition~\ref{def_stat_point}.
This result directly adapts existing analyses~\citep{ahn2024general, zhang2024random} and highlights the role of norm control in achieving optimal guarantees.

\begin{lemma}[Lemma 8 of~\cite{ahn2024general}, Lemma 3.2 of~\cite{zhang2024random}]
\label{lemma_bounding_second_term}
Using the notations of Proposition~\ref{prop_heavy_tail_conversion}. Let $\mathbb{E}[\rvy_{\tau}] = \widetilde{\rvx}_{\tau}$,  then $\forall t$, it holds that
\begin{align*}
\rho\mathbb{E}_{\tau}\mathbb{E}_{\rvy_{\tau}}\Vert\rvy_{\tau} - \widetilde{\rvx}_{\tau}\Vert^{2} &\le  \frac{2\rho\beta\sup_{t\in [T]}\mathbb{E}\left[\|\rvz_{t}\|^{2}\right]}{(1-\beta)^{2}}
\end{align*}
\end{lemma}
Thus, controlling $\|\rvz_t\|$, for example via $D$-based clipping, directly controls the variance term in the stationarity measure

\subsection{Optimization guarantee after conversion}

Combining Proposition~\ref{prop_heavy_tail_conversion} and Lemma~\ref{lemma_bounding_second_term}, we obtain the following $(\rho,\epsilon)$-stationarity guarantee under heavy-tailed noise:

\begin{align}
\label{eq_stationary_point_guarantee}
&\mathbb{E}_{\tau}\Vert\nabla F(\widetilde{\rvx}_{\tau})\Vert^{[\rho]} \le 
2(1-\beta + \frac{\beta}{T})(1 - \beta)^{-\frac{1}{p}}\sigma \nonumber\\
& + \frac{1}{DT}\left((1-\beta)\sum_{t=1}^{T}\mathbb{E}\mathrm{Regret}^{[\beta]}_t(\rvu_t) + \beta\mathbb{E}\mathrm{Regret}^{[\beta]}_T(\rvu_T)\right) \nonumber\\
&\qquad + \frac{\Delta}{DT} + \frac{2\rho\beta\sup_{t\in [T]}\mathbb{E}\left[\|\rvz_{t}\|^{2}\right]}{(1-\beta)^{2}}
\end{align}

From an algorithmic perspective, an effective method should:
(i) achieve small discounted regret
and (ii) control the magnitude of $\rvz_t$.

With this heavy-tailed discounted-to-nonconvex conversion framework in place, our next focus is to establish a regret analysis for Adam under heavy-tailed noise.
This constitutes the main technical contribution of the paper.
\begin{algorithm*}[!ht]
   \caption{Online learner: Adam (vector-form)}
   \label{alg_adam_learner}
\begin{algorithmic}[1]
   \STATE {\bfseries Input:} 
   $\beta \triangleq \beta_{1},\beta_{2}\in(0,1)$,
   step-size sequence $\{\alpha_t\}$,
   learning-rate sequence $\{\eta_t\}_{t\in[T]}$,
   and $\epsilon>0$.
   \FOR{$t=1$ {\bfseries to} $T$}
   \STATE
   $
   \rvz_{t}
   =
   \argmin_{\rvz\in\mathbb{R}^{d}}
   \frac{\|\rvz\|^{2}}{2\alpha_{t}}
   +\sum_{s=1}^{t-1}
   \left\langle
   \beta_{1}^{-s}\rvv_{s},\,\rvz
   \right\rangle = -\alpha_{t}\sum_{s=1}^{t-1}\beta_{1}^{-s}\rvv_{s}
   $$^{[a, b]}$
   \STATE Observe loss $\ell_t(\cdot)=\langle \rvv_t,\cdot\rangle$
   \ENDFOR
\end{algorithmic}
\hrule
\footnotesize
$^{[a]}$ Selecting $\alpha_{t}\triangleq  \eta_t  \frac{1 - \beta_{1}}{1 - \beta_{1}^{t-1}}\frac{\beta_{1}^{t-1}}{\sqrt{
\frac{1-\beta_{2}}{1-\beta_{2}^{t-1}}
\sum_{s=1}^{t-1}\beta_{2}^{t-1-s}\|\rvv_{s}\|^{2}}
+\epsilon}$, the incremental $\rvz_{t}$ can be formulated as
\[
\boxed{
\rvz_{t}  = -\eta_{t}\frac{\frac{1-\beta_{1}}{1-\beta_{1}^{t-1}}
\sum_{s=1}^{t-1}\beta_{1}^{t-1-s}\rvv_{s}}
{\sqrt{
\frac{1-\beta_{2}}{1-\beta_{2}^{t-1}}
\sum_{s=1}^{t-1}\beta_{2}^{t-1-s}\|\rvv_{s}\|^{2}}
+\epsilon}
\xRightarrow[\text{More familiar Adam update}]{\text{Rolling as EMA}}
\begin{cases}
\rvm_{t}=\beta_{1}\rvm_{t-1}+(1-\beta_{1})\rvv_{t-1} \\
v_{t}= \beta_{2}v_{t-1}+(1-\beta_{2})\|\rvv_{t-1}\|^{2} \\
\rvz_{t}=-\eta_{t}\frac{\rvm_{t}/(1-\beta_{1}^{t-1})}{\sqrt{v_{t}/(1-\beta_{2}^{t-1})}+\epsilon}.\\
\end{cases}
}
\]

$^{[b]}$ For the notation simplification purpose, we further define
\[
\boxed{
\rvz_{t} \triangleq -\frac{\eta_{t}\,\rva_{t}}{\sqrt{b_{t}}+\epsilon} \quad \text{where}\quad \rva_{t}
=
\frac{1-\beta_{1}}{1-\beta_{1}^{t-1}}
\sum_{s=1}^{t-1}\beta_{1}^{t-1-s}\rvv_{s}, \quad b_{t}
=
\frac{1-\beta_{2}}{1-\beta_{2}^{t-1}}
\sum_{s=1}^{t-1}\beta_{2}^{t-1-s}\|\rvv_{s}\|^{2}.
}
\]
\end{algorithm*}

\section{Algorithm: Adam}
\label{sec_alg}

In this section, we present Adam from an online-learning perspective.
Algorithm~\ref{alg_adam_learner} gives an equivalent formulation of Adam as a Follow-the-Regularized-Leader (FTRL) procedure with exponentially weighted costs and a quadratic regularizer.
Throughout the paper, we analyze this \emph{vector} version of Adam; an extension to coordinate-wise variants is discussed later.

With the choice of $\alpha_t$ in Algorithm~\ref{alg_adam_learner}$^{[a]}$, the output admits the algebraically equivalent form to the standard vector-form Adam update.
For later use, we adopt the shorthand $\rvz_t = -\eta_t \rva_t/(\sqrt{b_t}+\epsilon)$ defined in Algorithm~\ref{alg_adam_learner}$^{[b]}$.

\noindent
\textbf{Notation.}
We summarize the key quantities used throughout the analysis:
\begin{itemize}
    \item $\rvz_{t}$: Incremental / output of online learner Adam
    \item $\{\alpha_t\}$: step-size sequence
    \item $\{\eta_t\}$: learning-rate sequence
\end{itemize}

\noindent
\textbf{Organization.}
This section states three preparatory components used in the discounted regret analysis:
heavy-tail moment controls for Adam's adaptive denominator and the discounted cost norms that appear in the regret bound, a $\beta_2$-selection rule that yields a deterministic bound on $\|\rvz_t\|$, and monotonicity of effective step size $\alpha_{t}$.

\subsection{Heavy-tail control of discounted norms}

We next record the heavy-tail moment controls needed in the regret analysis.
Unlike bounded-variance analyses, the quantity $\sqrt{b_{t+1}}$ and the discounted norm aggregate $\sqrt{\sum_{s=1}^t\beta^{t-s}\|\rvv_s\|^2}$ cannot be bounded by second moments.
The key observation is that their expectations can still be controlled under a bounded $p$-th moment assumption by using the concavity of $x^{p/2}$ for $p\le2$.
The resulting bounds quantify the heavy-tail cost through the factor $(1-\beta)^{-(1/p-1/2)}$, which is the price of replacing a second-moment assumption by a $p$-moment assumption.


\begin{restatable}[]{proposition}{PropHeavyTailControl}
\label{proposition_heavy_tail_control}
Under Assumption~\ref{ass_vector_assumption}, for every $t\in[T]$,
\[
\mathbb{E}\sqrt{b_{t+1}}
\le
G
+
\sigma
\left(\frac{2}{p}\right)^{1/p}
(1-\beta_2)^{-\left(\frac1p-\frac12\right)} .
\]
\end{restatable}
The proof of Proposition~\ref{proposition_heavy_tail_control} is deferred to Appendix~\ref{append_sec_alg}.

\begin{corollary}
\label{corollary_cost_verctor_tail_norm_bound}
Assume Assumption~\ref{ass_vector_assumption} holds. Then for any $\beta\in(0,1)$ and $t\in[T]$,
\begin{align*}
&\mathbb{E}\sqrt{\sum_{s=1}^{t}\beta^{t-s}\|\rvv_{s}\|^{2}} \\
&\qquad \le \sqrt{\frac{1 - \beta^{t}}{1 - \beta}}\left(G + \sigma\left(\frac{2}{p}\right)^{1/p}(1 - \beta)^{-\left(\frac1p-\frac12\right)}\right)
\end{align*}
\end{corollary}

\subsection{Deterministic output bounds and effective step size control}
\label{subsec_lr_scheduler}

In this subsection, we first state a simple parameter condition that yields a deterministic bound on $\|\rvz_t\|$.

\begin{remark}[$\beta_{2}$ and $C(\beta)$ selection strategy]
\label{remark_beta2_Cbeta}
Fix $\beta_1 \in (0,1)$. Choose any $\beta_2 \in [\beta_{1},1)$ and a constant $C(\beta)\ge 1$ such that
$
1 - \beta_{1} = C(\beta)(1 - \beta_{2}).
$
For an integer-valued choice, it suffices to take
$
C(\beta)
\ge
\left\lceil \frac{1 - \beta_{1}}{1 - \beta_{2}} \right\rceil.
$
\end{remark}

The following proposition (Proof detailed in Appendix~\ref{append_sec_alg}) shows that this mild relation between $\beta_1$ and $\beta_2$ implies a deterministic norm bound for Adam's output, which will be crucial for our in-expectation analysis. 

\begin{restatable}[]{proposition}{PropStaticIncrementalBound}
\label{prop_staticincremental_bound}
Under Remark~\ref{remark_beta2_Cbeta}, Adam's output satisfies the deterministic bound
$
\|\rvz_t\| \;\le\; \eta_t\sqrt{C(\beta)}, \qquad \forall t\ge 1.
$
\end{restatable}

It is also worth noting that Remark~\ref{remark_beta2_Cbeta} and Proposition~\ref{prop_staticincremental_bound} allow $C(\beta)$ to be small, or even equal to~1 when $\beta_2=\beta_1$

Additionally, the following Proposition~\ref{prop_non_increase_step_size} records the natural monotonicity of the resulting effective step size.
Throughout the paper, we take a constant base step size, i.e., $\eta_t \equiv \eta$. 
\begin{restatable}[]{proposition}{PropNonincreasingStepSize}
\label{prop_non_increase_step_size}
Let $\beta_{1} \le \beta_{2}$, use $\eta_{s} = \eta$, and initialize $\alpha_{1}:=\alpha_{2}$. Then $\alpha_{s} = \frac{\eta(1 - \beta_{1})\beta_{1}^{s-1}}{(1 - \beta_{1}^{s-1})(\sqrt{b_{s}} + \epsilon)}, s\ge 2$, is non-increasing: $\alpha_{s}\ge \alpha_{s+1}$ for all $s \ge 1$.
\end{restatable}

\section{Regret and Optimization Guarantees}
\label{sec_guarantees}

This section presents our main performance guarantees for Adam under heavy-tailed noise.
We first establish a discounted regret bound for the Adam online learner.
We then combine this regret bound with the heavy-tailed discounted-to-nonconvex conversion (Section~\ref{sec_dtnc}) to obtain convergence guarantees to $(\rho,\epsilon)$-stationary points.
Finally, we show how the guarantee improves when the domain radius $D$ is available to the algorithm and is used to control the online-learner output.

\subsection{Discounted Regret under Heavy Tails}
\label{sec_regret}

We analyze Adam as an online learner receiving discounted linear losses and producing increments $\rvz_t$ as in Algorithm~\ref{alg_adam_learner}. The comparator radius $D$ specifies the regret benchmark; in the Adam result, $D$ is not used to modify the algorithm.
Our goal is to bound the discounted regret
\[
\mathbb{E}\left[\mathrm{Regret}^{[\beta]}_{t}(\rvu)\right]
=
\mathbb{E}\left[\sum_{s=1}^{t}\langle \beta_{1}^{t-s}\rvv_{s},\, \rvz_{s}-\rvu\rangle\right].
\]
The result of the regret bound under heavy-tailed noise is summarized in Theorem~\ref{theorem_regret_bound}, together with its supporting Proposition~\ref{prop_most_difficult_one}.

\begin{restatable}[]{proposition}{LemmaMostDifficultOne}
\label{prop_most_difficult_one}
Under Assumption~\ref{ass_vector_assumption}, for $t \ge T_{0}: = 2+\frac{\log(1/(1 - \beta_{1}))}{\log(\sqrt{\beta_{2}}/\beta_{1})}$, it holds that
\begin{align*}
&\mathbb{E}\left[
\sum_{s=1}^{t}\frac{\alpha_s}{2}\beta_1^{t-2s}\|\overline{\rvv}_s\|^2\right] \\
&\qquad \le \eta\frac{2\sqrt{\beta_2}(1-\beta_1)}{\beta_1(1-\beta_2)}\left(G + 2\sigma(1 - \beta_{2})^{-\left(\frac1p-\frac12\right)}\right)
\end{align*}
where $\overline{\rvv}_s:= \mathrm{clip}_{\sqrt{\sum_{n=1}^{s-1}\beta_2^{2s-2n}\|\rvv_n\|^2}}(\rvv_{s})$.
\end{restatable}
The proof of Proposition~\ref{prop_most_difficult_one} is deferred to Appendix~\ref{append_sec_guarantees}.

\begin{restatable}[]{theorem}{RegretStaticAdamLearner}
\label{theorem_regret_bound}
Assume Assumption~\ref{ass_vector_assumption} holds, and employ the $\beta_{2},C(\beta)$ selection strategy in Remark~\ref{remark_beta2_Cbeta}.
For any $t\ge T_{0}$ and any comparator $\rvu\in\mathcal{D}$ with $\|\rvu\|\le D$, if the online learner is Adam with $\eta = \frac{D}{(1 - \beta_{1})^{1/4}}$, then the discounted regret satisfies
\[
\mathbb{E}\left[\text{Regret}^{[\beta]}_{t}(\rvu) \right] \le \frac{C(R)D\left(G+2\sigma(1-\beta_2)^{-\left(\frac1p-\frac12\right)} + \epsilon\right)}{(1 - \beta_{1})^{3/4}}.
\]
where $C(R): = \max\big(\frac{1}{2} + \frac{2\sqrt{\beta_2}(1-\beta_1)^{3/2}}{\beta_1(1-\beta_2)} + \sqrt{8C(\beta)}, \frac{1}{2} + \frac{2\sqrt{\beta_2}(1-\beta_1)}{\beta_1(1-\beta_2)} + \sqrt{8}\big)$.
\end{restatable}

\begin{proof}
We first establish the deterministic intermediate regret bound in \eqref{eq_important_deter_result}, inspired by techniques presented in~\citet{orabona2019modern, ahn2024adam, ahn2024understanding, Tim2021Why}, and then control each component appearing in this bound.

\paragraph{Part 1. Intermediate result (deterministic bound) for $\text{Regret}^{[\beta]}_{t}(\rvu)$.}
Define $F_s(\rvz):=\frac{1}{2\alpha_s}\|\rvz\|^2
+\sum_{n=1}^{s-1}\langle\beta_1^{-n}\rvv_n,\rvz\rangle$
and $\alpha_s:=\eta_s\frac{(1-\beta_1)\beta_1^{s-1}}
{(1-\beta_1^{s-1})(\sqrt{b_s}+\epsilon)}$,
so
\[
  \rvz_s=\argmin F_s(\rvz)
  =
  -\alpha_s\sum_{n=1}^{s-1}\beta_1^{-n}\rvv_n
\]
is the Adam update.

Using $\sum(F_s(\rvz_s)-F_{s+1}(\rvz_{s+1})) =F_1(\rvz_1)-F_{t+1}(\rvz_{t+1})$,
$F_1(\rvz_1)=\min_{\rvx}\frac{1}{2\alpha_1}\|\rvx\|^2$, and $F_{t+1}(\rvz_{t+1})\le F_{t+1}(\rvu)$:
\begin{align*}
&\sum_{s=1}^{t}\langle\beta_1^{-s}\rvv_s,\rvz_s-\rvu\rangle\\
&\qquad \le \frac{\|\rvu\|^2}{2\alpha_{t+1}} +\sum_{s=1}^t\left[
     \underbrace{F_s(\rvz_s)-F_{s+1}(\rvz_{s+1})
       +\langle\beta_1^{-s}\rvv_s,\rvz_s\rangle}_{\text{Term B}}\right].
\end{align*}

Using $F_s(\rvz_{s+1})-F_{s+1}(\rvz_{s+1}) =-\langle\beta_1^{-s}\rvv_s,\rvz_{s+1}\rangle +\frac{1}{2\alpha_s}\|\rvz_{s+1}\|^2 -\frac{1}{2\alpha_{s+1}}\|\rvz_{s+1}\|^2$ and $\alpha_{s} \ge \alpha_{s+1}$ (Proposition~\ref{prop_non_increase_step_size}),
\begin{align*}
\text{Term B}
&\le
F_s(\rvz_s)-F_s(\rvz_{s+1})
+
\langle\beta_1^{-s}\rvv_s, \rvz_s-\rvz_{s+1}\rangle\\
&=
\underbrace{F_s(\rvz_s)-F_s(\rvz_{s+1})
+
\langle\beta_1^{-s}\overline{\rvv}_s, \rvz_s-\rvz_{s+1}\rangle}_{\text{B.1}} \\
&\qquad +
\underbrace{\langle\beta_1^{-s}\rvv_s - \beta_1^{-s}\overline{\rvv}_s, \rvz_s-\rvz_{s+1}\rangle}_{\text{B.2}}
\end{align*}
where $\overline{\rvv}_s:= \text{clip}_{\sqrt{\sum_{n=1}^{s-1}\beta_2^{2s-2n}\|\rvv_n\|^2}}(\rvv_{s})$.

\paragraph{Term~B.1:}
\begin{itemize}
\item $F_s(\rvz_s)-F_s(\rvz_{s+1}) \le -\frac{\|\rvz_s-\rvz_{s+1}\|^2}{2\alpha_s}$
by strong convexity and the variational inequality for the minimizer $\rvz_s = \argmin F_s(\rvz)$. \hfill(R1)
\item Employing Young's inequality,
$\langle\beta_1^{-s}\overline{\rvv}_s,\rvz_s-\rvz_{s+1}\rangle
\le \frac{\alpha_s}{2}\beta_1^{-2s}\|\overline{\rvv}_s\|^2
+\frac{\|\rvz_s-\rvz_{s+1}\|^2}{2\alpha_s}$.\hfill(R2)
\end{itemize}
Combining (R1)+(R2):
$\text{Term B.1 }\le \frac{\alpha_s}{2}\beta_1^{-2s}\|\overline{\rvv}_s\|^2$.

\paragraph{Term~B.2:}
\begin{align*}
    &\underbrace{\sum_{s=1}^{t}\beta_{1}^{-s}\langle\rvv_{s} - \overline{\rvv}_s, \rvz_{s} - \rvz_{s+1}\rangle}_{\text{Part B.2}} \\
    & \le \sum_{s=1}^{t}\beta_{1}^{-s}\|\rvv_{s} - \text{clip}_{\sqrt{\sum_{n=1}^{s-1}\beta_2^{2s-2n}\|\rvv_n\|^2}}(\rvv_{s})\| \|\rvz_{s} - \rvz_{s+1}\|\\
    & \stackrel{(a)}{\le} 2\max_{s\in [t]}\|\rvz_{s}\|\sum_{s =1}^{t}\beta_{1}^{-s}\left(\|\rvv_{s}\| - \sqrt{\sum_{n=1}^{s-1}\beta_2^{2s-2n}\|\rvv_n\|^2}\right)_{+} \\
    & \stackrel{(b)}{\le} 2\max_{s\in [t]}\|\rvz_{s}\|\sum_{s =1}^{t}\left(\beta_{1}^{-s}\|\rvv_{s}\| - \sqrt{\sum_{n=1}^{s-1}\beta_1^{-2n}\|\rvv_n\|^2}\right)_{+}\\
    & \le 2\max_{s\in [t]}\|\rvz_{s}\|\sum_{s =1}^{t}\left(\sqrt{\sum_{n=1}^{s}\beta_{1}^{-2n}\|\rvv_{n}\|^{2}} - \sqrt{\sum_{n=1}^{s-1}\beta_{1}^{-2n}\|\rvv_{n}\|^{2}}\right)\\
    & \le 2\max_{s\in [t]}\|\rvz_{s}\| \sqrt{\sum_{s=1}^{t}\beta_{1}^{-2s}\|\rvv_{s}\|^{2}},
\end{align*}
where $(a)$ uses the identity $\|\rvv-\mathrm{clip}_{R}(\rvv)\|=(\|\rvv\|-R)_+$, and $(b)$ uses $\beta_{1}\le\beta_{2}$.

Combining the results of Term B.1 and Term B.2, it suffices to have
\begin{align*}
\sum_{s=1}^{t}\langle\beta_1^{-s}\rvv_s,\rvz_s-\rvu\rangle
&\le \frac{\|\rvu\|^2}{2\alpha_{t+1}}
 + \sum_{s=1}^{t}\frac{\alpha_s}{2}\beta_1^{-2s}\|\overline{\rvv}_s\|^2 \\
& \qquad + 2\max_{s\in [t]}\|\rvz_{s}\| \sqrt{\sum_{s=1}^{t}\beta_{1}^{-2s}\|\rvv_{s}\|^{2}}.
\end{align*}
Multiplying $\beta_{1}^{t}$ on both sides gives 
\begin{align}
\label{eq_important_deter_result}
\mathrm{Regret}_t^{[\beta]}(\rvu) &\le\frac{\beta_1^t\|\rvu\|^2}{2\alpha_{t+1}}+ \sum_{s=1}^{t}\frac{\alpha_s}{2}\beta_1^{t-2s}\|\overline{\rvv}_s\|^2 \nonumber\\
&\quad + 2\max_{s\in [t]}\|\rvz_{s}\| \sqrt{\sum_{s=1}^{t}\beta_{1}^{2t-2s}\|\rvv_{s}\|^{2}}
\end{align}
where $\overline{\rvv}_s:= \mathrm{clip}_{\sqrt{\sum_{n=1}^{s-1}\beta_2^{2s-2n}\|\rvv_n\|^2}}(\rvv_{s})$. 

\paragraph{Part 2. Final regret bound.}
Taking expectation over \eqref{eq_important_deter_result},
\begin{align*}
&\mathbb{E}\left[\text{Regret}^{[\beta]}_{t}(\rvu) \right] \nonumber\\
&\le \frac{D^{2}}{2\eta(1 - \beta_{1})}\left(G
+
2\sigma
(1-\beta_2)^{-\left(\frac1p-\frac12\right)}\right) \nonumber\\
& \qquad + \eta\frac{2\sqrt{\beta_2}(1-\beta_1)}{\beta_1(1-\beta_2)}\left(G + 2\sigma(1 - \beta_{2})^{-\left(\frac1p-\frac12\right)}\right)\nonumber\\
& \qquad + 2\eta\sqrt{C(\beta)}\sqrt{\frac{1}{1 - \beta_{1}^{2}}}\left(G + 2\sigma(1 - \beta_{1}^{2})^{-\left(\frac1p-\frac12\right)}\right)
\end{align*}
Here, the first term on the right-hand side uses Proposition~\ref{proposition_heavy_tail_control}, the second term uses Proposition~\ref{prop_most_difficult_one}, and the last term uses Corollary~\ref{corollary_cost_verctor_tail_norm_bound} and the deterministic increment bound in Proposition~\ref{prop_staticincremental_bound}.

The above inequality can be further reformulated as
\begin{align}
\label{eq_regret_before_incremental_operation}
&\mathbb{E}\left[\text{Regret}^{[\beta]}_{t}(\rvu) \right] \nonumber\\
&\le \left(\frac{D^{2}}{2\eta(1 - \beta_{1})} +  \eta\frac{2\sqrt{\beta_2}(1-\beta_1)}{\beta_1(1-\beta_2)} + \eta\sqrt{C(\beta)}\sqrt{\frac{8}{1 - \beta_{1}}}\right)\nonumber\\
&\qquad \times\left(G
+
2\sigma
(1-\beta_2)^{-\left(\frac1p-\frac12\right)} + \epsilon\right).
\end{align}
Setting $\eta = \frac{D}{(1 - \beta_{1})^{1/4}}$ gives
\begin{align*}
\mathbb{E}\left[\text{Regret}^{[\beta]}_{t}(\rvu)\right] &\le \frac{D\left(G+2\sigma(1-\beta_2)^{-\left(\frac1p-\frac12\right)} + \epsilon\right)}{(1 - \beta_{1})^{3/4}}\\
&\quad \times \left(\frac{1}{2} + \frac{2\sqrt{\beta_2}(1-\beta_1)^{3/2}}{\beta_1(1-\beta_2)} + \sqrt{8C(\beta)}\right),
\end{align*}
which concludes the proof.
\end{proof}

Next, we consider the regret bound when the learner output is explicitly controlled by the known radius $D$.
\begin{corollary}
\label{corollary_inc_norm_bound}
Assume Assumption~\ref{ass_vector_assumption} holds.
If the learner is the constrained-FTRL form of Adam over the ball $\{\|\rvz\|\le D\}$, equivalently the Adam increment is clipped as $\rvz\leftarrow\mathrm{clip}_{D}(\rvz)$, and $\eta = \frac{D}{\sqrt{1 - \beta_{1}}}$, then
\[
\mathbb{E}\left[\text{Regret}^{[\beta]}_{t}(\rvu) \right] \le \frac{C(R)D\left(G+2\sigma(1-\beta_2)^{-\left(\frac1p-\frac12\right)} + \epsilon\right)}{\sqrt{1 - \beta_{1}}}.
\]
where $C(R): = \max\big(\frac{1}{2} + \frac{2\sqrt{\beta_2}(1-\beta_1)^{3/2}}{\beta_1(1-\beta_2)} + \sqrt{8C(\beta)}, \frac{1}{2} + \frac{2\sqrt{\beta_2}(1-\beta_1)}{\beta_1(1-\beta_2)} + \sqrt{8}\big)$.
\end{corollary}
\begin{proof}
For the clipped method, $\rvz_s$ is the FTRL minimizer over the ball $\{\|\rvz\|\le D\}$, which is equivalent to clipping the unconstrained Adam increment. The same decomposition as in Part 1 of Theorem~\ref{theorem_regret_bound} applies, and we begin from \eqref{eq_regret_before_incremental_operation} but substitute deterministic increment bound term $\eta\sqrt{C(\beta)}$ as $D$
\begin{align*}
&\mathbb{E}\left[\text{Regret}^{[\beta]}_{t}(\rvu) \right] \\
&\le \left(\frac{D^{2}}{2\eta(1 - \beta_{1})} +  \eta\frac{2\sqrt{\beta_2}(1-\beta_1)}{\beta_1(1-\beta_2)} + \sqrt{\frac{8}{1 - \beta_{1}}}D\right)\\
&\qquad \times\left(G
+
2\sigma
(1-\beta_2)^{-\left(\frac1p-\frac12\right)} + \epsilon\right).
\end{align*}
Setting $\eta = \frac{D}{\sqrt{1 - \beta_{1}}}$ gives
\begin{align*}
\mathbb{E}\left[\text{Regret}^{[\beta]}_{t}(\rvu) \right] &\le \frac{D\left(G+2\sigma(1-\beta_2)^{-\left(\frac1p-\frac12\right)} + \epsilon\right)}{\sqrt{1 - \beta_{1}}}\\
&\qquad \times\left(\frac{1}{2} + \frac{2\sqrt{\beta_2}(1-\beta_1)}{\beta_1(1-\beta_2)} + \sqrt{8}\right).
\end{align*}
This concludes the proof.
\end{proof}

In the stationarity results below, $\epsilon$ denotes the target accuracy. The numerical stabilizer in the Adam denominator is treated as a fixed small constant and is omitted from the displayed convergence rates to avoid overloading notation.

\subsection{Convergence Guarantee under Heavy Tails}
\label{sec_convergence}

We now combine the discounted regret bound (Theorem~\ref{theorem_regret_bound}) with the heavy-tailed discounted-to-nonconvex conversion (Section~\ref{sec_dtnc}) to obtain a $(\rho,\epsilon)$-stationarity guarantee for Adam.
At a high level, the conversion inequality~\eqref{eq_stationary_point_guarantee} mainly involves two terms:
(i) a regret term,
and (ii) a variance term controlled by $\sup_t\mathbb{E}\|\rvz_t\|^2$.
Theorem~\ref{theorem_regret_bound} controls (i), Proposition~\ref{prop_staticincremental_bound} controls (ii), and we choose $(\beta_1,D,T)$ to balance all contributions.

\begin{restatable}{theorem}{TheoremFinalConvergence}
\label{theorem_final_convergence}
Assume $F$ satisfies Assumption~\ref{ass_func} and fix any $\rho>0$.
Run Algorithm~\ref{alg_dtnc} with online learner Adam as in Algorithm~\ref{alg_adam_learner}.
Then, for any $p\in (\frac{4}{3}, 2]$, there exists a random index $\tau\in[T]$ such that
\begin{align*}
\mathbb{E}_{\tau}\left[\Vert\nabla F(\widetilde{\rvx}_{\tau})\Vert^{[\rho]}\right] \leq \left(
8+C(R)+
2C(R)C(\beta)^{\frac1p-\frac12}
\right)\epsilon,
\end{align*}
where $\tau$ is a random index among $[T]$ such that $\mathbb{P}(\tau = t) = \begin{cases}
\frac{1-\beta_{1}^{t}}{T} & \text{if}\quad t = 1,\cdots,T-1,\\
\frac{1 - \beta_{1}^{T}}{(1-\beta_{1})T} & \text{if} \quad t = T
\end{cases}$.
Moreover, the $\rvy_t$ is randomly distributed over $\{\rvx_{s}\}_{s=1}^{t}$ as $\mathbb{P}(\rvy_t=\rvx_s) \;=\; \frac{(1-\beta_{1})\beta_{1}^{t-s}}{1-\beta_{1}^t}, s=1,\ldots,t$. 

In particular, it suffices to choose: 

$\beta_{1} 
=
1-
\left(
\frac{\epsilon}{G+\sigma}
\right)^{\frac{4p}{3p-4}}$, 
$\beta_2$ so that
\(
1-\beta_1 = C(\beta)(1-\beta_{2})
\)
for a fixed constant $C(\beta)\ge1$
,
$D=
\frac{(1-\beta_1)^{5/4}\epsilon^{1/2}}
{C(\beta)\rho^{1/2}}$, and 

$T = \mathcal{O}\Big(\max \Big\{\left(C(R)G+\sigma\right)^{\frac{8p- 4}{3p-4}}
\epsilon^{-\frac{8p -4}{3p-4}}\log\frac{G + \sigma}{\epsilon}, \\\Delta\rho^{\frac{1}{2}}(C(R)G+\sigma)^{\frac{5p}{3p-4}}\epsilon^{-\left(\frac{5p}{3p-4}+\frac32\right)} \Big\}\Big)$.

Moreover, in the special case $p=2$, the bound further reduces to
$T = \mathcal{O}\Big(\max \Big\{\left(C(R)G+\sigma\right)^{6}
\epsilon^{-6}\log\frac{G + \sigma}{\epsilon}, \\\Delta\rho^{\frac{1}{2}}(C(R)G+\sigma)^{5}\epsilon^{-\frac{13}{2}} \Big\}\Big)$.
\end{restatable}
For readability, the full max-form conditions on $T$, including lower-order and burn-in terms, and the proof of Theorem~\ref{theorem_final_convergence}, are given in Appendix~\ref{append_sec_guarantees}.

\paragraph{Discussion of rates.}
Theorem~\ref{theorem_final_convergence} implies that Adam converges in expectation to a $(\rho, \epsilon)$-stationary with an $\mathcal{O}\left(\epsilon^{-3}\right)$ gap compared to the optimal rate under the same framework~\citep{zhang2024random, liu2025online}, in the bounded-variance case $p=2$.
This provides a rigorous convergence guarantee for the \emph{plain} Adam update, albeit at a sub-optimal rate.

\subsection{Guarantee with known domain radius}
\label{sec_gurantee_with_D}

In this case, we use the constrained-FTRL version of Adam over $\{\|\rvz\|\le D\}$, equivalently clipping the online-learner output as $\rvz_t \leftarrow \mathrm{clip}_D(\rvz_t)$.
This directly controls the variance term in the conversion bound and improves the overall iteration complexity.

\begin{restatable}{corollary}{CorollaryWithKnowledgeD}
\label{corollary_with_knowledge_D}
Use the same notation and assumptions as in Theorem~\ref{theorem_final_convergence}.
Suppose the Adam online learner is run in constrained form over $\{\|\rvz\|\le D\}$, equivalently replacing $\rvz_t$ by $\mathrm{clip}_{D}(\rvz_t)$ in the conversion algorithm.
Then, for any $p \in (1,2]$, the method guarantees a $(\rho, \epsilon)$-stationary point 
when choosing 

$\beta_1
=
1-
\left(
\frac{\epsilon}{G+\sigma}
\right)^{\frac{p}{p-1}}$, $\beta_2$ so that
\(
1-\beta_1 = C(\beta)(1-\beta_{2})
\)
for a fixed constant $C(\beta)\ge1$,
$D=
\frac{(1-\beta_1)\epsilon^{1/2}}
{\rho^{1/2}}$, and

$T = \mathcal{O}\Big(\max \Big\{\left(G+\sigma\right)^{\frac{p}{p-1} + 1}
\epsilon^{-\left(\frac{p}{p-1}+1\right)}\log\frac{G + \sigma}{\epsilon}, \\\Delta\rho^{\frac{1}{2}}
\left(C(R)G+\sigma\right)^{\frac{p}{p-1}}
\epsilon^{-\left(\frac{p}{p-1}+\frac32\right)} \Big\}\Big)$.

Moreover, in the special case $p=2$, the bound further reduces to
$T = \mathcal{O}\Big(\max \Big\{\left(G+\sigma\right)^{3}
\epsilon^{-3}\log\frac{G + \sigma}{\epsilon}, \\\Delta\rho^{\frac{1}{2}}(G +\sigma)^{2}\epsilon^{-\frac{7}{2}} \Big\}\Big)$.
\end{restatable}
The proof of Corollary~\ref{corollary_with_knowledge_D} is deferred to Appendix~\ref{append_sec_guarantees}.

\begin{remark}
Corollary~\ref{corollary_with_knowledge_D} matches the known optimal heavy-tail exponent for stochastic non-smooth non-convex optimization under the same conversion framework, up to constants and lower-order terms~\citep{zhang2024random, ahn2024adam, liu2025online}.

Moreover, under additional smoothness assumptions, these rates translate into $\mathcal{O}(\epsilon^{-4})$ complexity for standard first-order stationarity, and admit natural extensions to coordinate-wise versions of Adam targeting $L_1$-stationarity under standard coordinate-wise assumptions; we omit the details since the extension follows existing arguments~\citep{ahn2024adam}.
\end{remark}

\section{Conclusion and Discussion}
In this work, we established convergence guarantees for Adam under heavy-tailed noise.
Our analysis proceeds by (i) extending the discounted-to-nonconvex conversion framework to heavy-tailed martingale-difference noise, and (ii) developing a discounted regret analysis for exact vector-form Adam. Together, these ingredients yield $(\rho,\epsilon)$-stationarity guarantees under only bounded $p$-th moments.
For plain Adam, the current analysis applies for $p>4/3$ and gives the leading exponent $\frac{5p}{3p-4}+\frac32$ in the target accuracy.
When the domain radius is known and used to control the online-learner output, the exponent improves to the optimal heavy-tail rate $\frac{p}{p-1}+\frac32$ for all $p\in(1,2]$.
In particular, for $p=2$, the plain-Adam bound scales as $\mathcal{O}(\epsilon^{-13/2})$, whereas the known-radius version scales as $\mathcal{O}(\epsilon^{-7/2})$.

\paragraph{On the selection of $\beta_{2}$ and numerical stabilizer.}
Our regret analysis uses a mild relation between $\beta_1$ and $\beta_2$ ($\beta_{1}\le \beta_{2}$ in Remark~\ref{remark_beta2_Cbeta}), written as $\frac{1-\beta_1}{1-\beta_2}\le C(\beta)$, to obtain deterministic control of the Adam increment.
This condition allows $C(\beta)$ to be a small constant and does not require the restrictive coupling.
The numerical stabilizer in Adam’s denominator is treated as a fixed small constant and is not required to scale with either $G$ or $\sigma$.

\paragraph{An open question.}
Corollary~\ref{corollary_with_knowledge_D} shows that explicit norm control of the increment suffices to close the optimality gap, yielding optimal iteration complexity.
This raises a natural question:
\emph{Can the plain Adam update, or a closely related variant such as AdamW, without access to $D$ for output clipping, still attain optimal convergence rates under heavy-tailed noise?}
Answering this question would clarify whether explicit norm control is fundamentally necessary, or merely an artifact of the current analytical techniques.

\bibliography{references}

@article{orabona2019modern,
  title={A modern introduction to online learning},
  author={Orabona, Francesco},
  journal={arXiv preprint arXiv:1912.13213},
  year={2019}
}

@article{zhang2019gradient,
  title={Why gradient clipping accelerates training: A theoretical justification for adaptivity},
  author={Zhang, Jingzhao and He, Tianxing and Sra, Suvrit and Jadbabaie, Ali},
  journal={arXiv preprint arXiv:1905.11881},
  year={2019}
}

@article{yu2026sign,
  title={Sign-Based Optimizers Are Effective Under Heavy-Tailed Noise},
  author={Yu, Dingzhi and Tao, Hongyi and Wan, Yuanyu and Luo, Luo and Zhang, Lijun},
  journal={arXiv preprint arXiv:2602.07425},
  year={2026}
}

@article{chezhegov2024clipping,
  title={Clipping improves adam-norm and adagrad-norm when the noise is heavy-tailed},
  author={Chezhegov, Savelii and Klyukin, Yaroslav and Semenov, Andrei and Beznosikov, Aleksandr and Gasnikov, Alexander and Horv{\'a}th, Samuel and Tak{\'a}{\v{c}}, Martin and Gorbunov, Eduard},
  journal={arXiv preprint arXiv:2406.04443},
  year={2024}
}

@article{pinelis2015best,
  title={Best possible bounds of the von Bahr--Esseen type},
  author={Pinelis, Iosif},
  journal={Annals of Functional Analysis},
  volume={6},
  number={4},
  pages={1--29},
  year={2015},
  publisher={Tusi Mathematical Research Group}
}

@article{ahn2024general,
  title={General framework for online-to-nonconvex conversion: Schedule-free SGD is also effective for nonconvex optimization},
  author={Ahn, Kwangjun and Magakyan, Gagik and Cutkosky, Ashok},
  journal={arXiv preprint arXiv:2411.07061},
  year={2024}
}

@article{Tim2021Why,
  author = {Tim, van Erven},
  title = {Why FTRL Is Better than Online Mirror Descent},
  year = 2021,
  journal = {https://www.timvanerven.nl/blog/ftrl-vs-omd/},
  urldate = {2021-09-21}
}

@article{ahn2024adam,
  title={Adam with model exponential moving average is effective for nonconvex optimization},
  author={Ahn, Kwangjun and Cutkosky, Ashok},
  journal={arXiv preprint arXiv:2405.18199},
  year={2024}
}

@article{ahn2024understanding,
  title={Understanding Adam Optimizer via Online Learning of Updates: Adam is FTRL in Disguise},
  author={Ahn, Kwangjun and Zhang, Zhiyu and Kook, Yunbum and Dai, Yan},
  journal={arXiv preprint arXiv:2402.01567},
  year={2024}
}

@inproceedings{cutkosky2023optimal,
  title={Optimal stochastic non-smooth non-convex optimization through online-to-non-convex conversion},
  author={Cutkosky, Ashok and Mehta, Harsh and Orabona, Francesco},
  booktitle={International Conference on Machine Learning},
  pages={6643--6670},
  year={2023},
  organization={PMLR}
}

@article{zhang2024random,
  title={Random Scaling and Momentum for Non-smooth Non-convex Optimization},
  author={Zhang, Qinzi and Cutkosky, Ashok},
  journal={arXiv preprint arXiv:2405.09742},
  year={2024}
}

@article{ilboudo2023adaterm,
  title={Adaterm: Adaptive t-distribution estimated robust moments for noise-robust stochastic gradient optimization},
  author={Ilboudo, Wendyam Eric Lionel and Kobayashi, Taisuke and Matsubara, Takamitsu},
  journal={Neurocomputing},
  volume={557},
  pages={126692},
  year={2023},
  publisher={Elsevier}
}

@article{ahn2023linear,
  title={Linear attention is (maybe) all you need (to understand transformer optimization)},
  author={Ahn, Kwangjun and Cheng, Xiang and Song, Minhak and Yun, Chulhee and Jadbabaie, Ali and Sra, Suvrit},
  journal={arXiv preprint arXiv:2310.01082},
  year={2023}
}

@inproceedings{liu2023breaking,
  title={Breaking the lower bound with (little) structure: Acceleration in non-convex stochastic optimization with heavy-tailed noise},
  author={Liu, Zijian and Zhang, Jiawei and Zhou, Zhengyuan},
  booktitle={The Thirty Sixth Annual Conference on Learning Theory},
  pages={2266--2290},
  year={2023},
  organization={PMLR}
}

@article{zhang2022parameter,
  title={Parameter-free regret in high probability with heavy tails},
  author={Zhang, Jiujia and Cutkosky, Ashok},
  journal={Advances in Neural Information Processing Systems},
  volume={35},
  pages={8000--8012},
  year={2022}
}

@article{fradin2025tight,
  title={Tight Lower Bounds and Optimal Algorithms for Stochastic Nonconvex Optimization with Heavy-Tailed Noise},
  author={Fradin, Adrien and Sadiev, Abdurakhmon and Condat, Laurent and Richt{\'a}rik, Peter},
  journal={arXiv preprint arXiv:2512.18713},
  year={2025}
}

@article{sun2024gradient,
  title={Gradient normalization provably benefits nonconvex sgd under heavy-tailed noise},
  author={Sun, Tao and Liu, Xinwang and Yuan, Kun},
  journal={arXiv preprint arXiv:2410.16561},
  pages={5},
  year={2024}
}

@article{he2025complexity,
  title={Complexity of normalized stochastic first-order methods with momentum under heavy-tailed noise},
  author={He, Chuan and Lu, Zhaosong and Sun, Defeng and Deng, Zhanwang},
  journal={arXiv preprint arXiv:2506.11214},
  year={2025}
}

@article{robbins1951stochastic,
  title={A stochastic approximation method},
  author={Robbins, Herbert and Monro, Sutton},
  journal={The annals of mathematical statistics},
  pages={400--407},
  year={1951},
  publisher={JSTOR}
}

@article{zhang2020adaptive,
  title={Why are adaptive methods good for attention models?},
  author={Zhang, Jingzhao and Karimireddy, Sai Praneeth and Veit, Andreas and Kim, Seungyeon and Reddi, Sashank and Kumar, Sanjiv and Sra, Suvrit},
  journal={Advances in Neural Information Processing Systems},
  volume={33},
  pages={15383--15393},
  year={2020}
}

@inproceedings{garg2021proximal,
  title={On proximal policy optimization’s heavy-tailed gradients},
  author={Garg, Saurabh and Zhanson, Joshua and Parisotto, Emilio and Prasad, Adarsh and Kolter, Zico and Lipton, Zachary and Balakrishnan, Sivaraman and Salakhutdinov, Ruslan and Ravikumar, Pradeep},
  booktitle={International Conference on Machine Learning},
  pages={3610--3619},
  year={2021},
  organization={PMLR}
}

@inproceedings{battash2024revisiting,
  title={Revisiting the noise model of stochastic gradient descent},
  author={Battash, Barak and Wolf, Lior and Lindenbaum, Ofir},
  booktitle={International Conference on Artificial Intelligence and Statistics},
  pages={4780--4788},
  year={2024},
  organization={PMLR}
}

@article{nguyen2023improved,
  title={Improved convergence in high probability of clipped gradient methods with heavy tails},
  author={Nguyen, Ta Duy and Ene, Alina and Nguyen, Huy L},
  journal={arXiv preprint arXiv:2304.01119},
  year={2023}
}

@article{gorbunov2020stochastic,
  title={Stochastic optimization with heavy-tailed noise via accelerated gradient clipping},
  author={Gorbunov, Eduard and Danilova, Marina and Gasnikov, Alexander},
  journal={Advances in Neural Information Processing Systems},
  volume={33},
  pages={15042--15053},
  year={2020}
}

@article{goldstein1977optimization,
  title={Optimization of Lipschitz continuous functions},
  author={Goldstein, Allen A},
  journal={Mathematical Programming},
  volume={13},
  pages={14--22},
  year={1977},
  publisher={Springer}
}

@inproceedings{jordan2023deterministic,
  title={Deterministic nonsmooth nonconvex optimization},
  author={Jordan, Michael and Kornowski, Guy and Lin, Tianyi and Shamir, Ohad and Zampetakis, Manolis},
  booktitle={The Thirty Sixth Annual Conference on Learning Theory},
  pages={4570--4597},
  year={2023},
  organization={PMLR}
}

@inproceedings{tian2022finite,
  title={On the finite-time complexity and practical computation of approximate stationarity concepts of lipschitz functions},
  author={Tian, Lai and Zhou, Kaiwen and So, Anthony Man-Cho},
  booktitle={International Conference on Machine Learning},
  pages={21360--21379},
  year={2022},
  organization={PMLR}
}

@article{kingma2014adam,
  title={Adam: A method for stochastic optimization},
  author={Kingma, Diederik P},
  journal={arXiv preprint arXiv:1412.6980},
  year={2014}
}

@article{liu2025online,
  title={Online Convex Optimization with Heavy Tails: Old Algorithms, New Regrets, and Applications},
  author={Liu, Zijian},
  journal={arXiv preprint arXiv:2508.07473},
  year={2025}
}

@article{von1965inequalities,
  title={Inequalities for the rth absolute moment of a sum of random variables, 1 $\le$ r $\le$ 2},
  author={von Bahr, Bengt and Esseen, Carl-Gustav},
  journal={The Annals of Mathematical Statistics},
  pages={299--303},
  year={1965},
  publisher={JSTOR}
}

\clearpage

\onecolumn

\title{Supplementary Material}
\maketitle

\appendix
\section{Missing Proofs of Section~\ref{sec_dtnc}}
\label{append_sec_dtnc}

\LemmaDiscountedSumIID*
\begin{proof}
Using $\|\rvu_n\|\le D$ and Cauchy--Schwarz, we have
\[
\sum_{t=1}^{n}\beta^{n-t}\langle \boldsymbol{\xi}_t,\rvu_n\rangle
=
\Bigl\langle \sum_{t=1}^{n}\beta^{n-t}\boldsymbol{\xi}_t,\ \rvu_n\Bigr\rangle
\le
D\Bigl\|\sum_{t=1}^{n}\beta^{n-t}\boldsymbol{\xi}_t\Bigr\|.
\]
Taking expectation and applying Jensen's inequality (since $x\mapsto x^p$ is convex for $p\ge 1$) yields
\begin{align*}
\mathbb{E}\sum_{t=1}^{n}\beta^{n-t}\langle \boldsymbol{\xi}_t,\rvu_n\rangle
&\le
D\,\mathbb{E}\left[\Bigl\|\sum_{t=1}^{n}\beta^{n-t}\boldsymbol{\xi}_t\Bigr\|\right]
\le
D\Biggl(\mathbb{E}\left[\Bigl\|\sum_{t=1}^{n}\beta^{n-t}\boldsymbol{\xi}_t\Bigr\|^p\right]\Biggr)^{1/p}. \tag{R1}
\end{align*}

We now bound the $p$-moment of the weighted sum using the unbiasedness condition
$\mathbb{E}[\boldsymbol{\xi}_t\mid\mathcal{F}_{t-1}]=\mathbf{0}$ and $p\in(1,2]$.
Define $\rvd_t \triangleq \beta^{n-t}\boldsymbol{\xi}_t$. Since $\beta^{n-t}$ is deterministic,
$\mathbb{E}[\rvd_t\mid\mathcal{F}_{t-1}]=\beta^{n-t}\mathbb{E}[\boldsymbol{\xi}_t\mid\mathcal{F}_{t-1}]=\mathbf{0}$,
so $\{\rvd_t,\mathcal{F}_t\}$ is a martingale-difference sequence.

By a von Bahr--Esseen type inequality for martingale differences for $p\in(1,2]$~\citep{von1965inequalities, pinelis2015best}, there exists a constant $C_p\le 2$ such that
\begin{align*}
\mathbb{E}\Bigl\|\sum_{t=1}^{n}\beta^{n-t}\boldsymbol{\xi}_t\Bigr\|^p
=
\mathbb{E}\Bigl\|\sum_{t=1}^{n}\rvd_t\Bigr\|^p
&\le
C_p\sum_{t=1}^{n}\mathbb{E}\|\rvd_t\|^p
=
C_p\sum_{t=1}^{n}\beta^{p(n-t)}\,\mathbb{E}\|\boldsymbol{\xi}_t\|^p \\
&\le
C_p\,\sigma^p\sum_{k=0}^{n-1}\beta^{pk}
\le
\frac{C_p\,\sigma^p}{1-\beta^p}
\le
\frac{C_p\,\sigma^p}{1-\beta}. \tag{R2}
\end{align*}

Combining (R1) and (R2) gives
\begin{align*}
\mathbb{E}\sum_{t=1}^{n}\beta^{n-t}\langle \boldsymbol{\xi}_t,\rvu_n\rangle
&\le
D\cdot \frac{C_p^{1/p}\sigma}{(1-\beta^p)^{1/p}}
\le
D\cdot \frac{C_p^{1/p}\sigma}{(1-\beta)^{1/p}}
\le
2(1-\beta)^{-1/p}\sigma D,
\end{align*}
where the last inequality uses $C_p\le 2$.
\end{proof}

\begin{restatable}[]{lemma}{LemmaKeySingleInConversion}
\label{lem_key_single_in_conversion}
For $\beta \in (0,1)$, consider the iterates generated as per~\ref{alg_dtnc}. Then, $\forall n\in [T]$, we have
\begin{align*}
\mathbb{E}\left[\sum_{t=1}^{n}\beta^{n-t}\bigl(F(\rvx_t)-F(\rvx_{t-1})\bigr)\right]\le
-D\frac{1 - \beta^{n}}{1 - \beta}\mathbb{E}\Bigl\|\mathbb{E}_{\rvy_n}\nabla F(\rvy_n)\Bigr\| + 2(1 - \beta)^{-\frac{1}{p}}\sigma D + \mathbb{E}\mathrm{Regret}^{[\beta]}_n(\rvu_n),
\end{align*}
where $\rvy_n$ is randomly distributed over $\{\rvx_{t}\}_{t=1}^{n}$ as $\mathbb{P}(\rvy_n=\rvx_t) \;=\; \frac{(1-\beta)\beta^{n-t}}{1-\beta^n}, t=1,\ldots,n$.
\end{restatable}
\begin{proof}
Following the fact about exponential random variable, Lemma 3.1 in~\cite{zhang2024random}, $\mathbb{E}\bigl[F(\rvx_t)-F(\rvx_{t-1})\bigr]= \mathbb{E}\langle \nabla F(\rvx_t),\rvz_t\rangle$:
\begin{align}
\mathbb{E}\bigl[F(\rvx_t)-F(\rvx_{t-1})\bigr]
&=
\mathbb{E}\langle \nabla F(\rvx_t),\rvz_t\rangle
\nonumber\\
&=\mathbb{E}\langle \nabla F(\rvx_t),\rvu_n\rangle
+
\mathbb{E}\langle \nabla F(\rvx_t)-\rvg_t, \rvz_{t}-\rvu_n\rangle
+
\mathbb{E}\langle \rvg_t, \rvz_t-\rvu_n\rangle \nonumber\\
&=
\underbrace{\mathbb{E}\langle \nabla F(\rvx_t),\rvu_n\rangle}_{\text{\textcircled{1}}}
+
\underbrace{\mathbb{E}\langle \nabla F(\rvx_t)-\rvg_t,-\rvu_n\rangle}_{\text{\textcircled{2}}}
+
\underbrace{\mathbb{E}\langle \rvg_t, \rvz_t-\rvu_n\rangle}_{\text{\textcircled{3}}},
\label{eq:decomp_A}
\end{align}
where we used the fact $\mathbb{E}\langle \nabla F(\rvx_t)-\rvg_t, \rvz_t\rangle = 0$ for the last equality.

We bound the three pieces.

\medskip
\noindent\textcircled{1}:\;
Define
$
\rvu_n \;\triangleq\; -D\,\frac{\sum_{t=1}^{n}\beta^{n-t}\nabla F(\rvx_t)}{\left\|\sum_{t=1}^{n}\beta^{n-t}\nabla F(\rvx_t)\right\|},
$
then 
\begin{align*}
\mathbb{E}\sum_{t=1}^{n}\beta^{n-t}\langle \nabla F(\rvx_t),\rvu_n\rangle
&=
\mathbb{E}\Bigg\langle
\sum_{t=1}^{n}\beta^{n-t}\nabla F(\rvx_t),
\,-D\frac{\sum_{t=1}^{n}\beta^{n-t}\nabla F(\rvx_t)}{\left\|\sum_{t=1}^{n}\beta^{n-t}\nabla F(\rvx_t)\right\|}
\Bigg\rangle
\nonumber\\
&=
-D\mathbb{E}\Bigl\|\sum_{t=1}^{n}\beta^{n-t}\nabla F(\rvx_t)\Bigr\|\\
&\le
-D\frac{1 - \beta^{n}}{1 - \beta}\mathbb{E}\Bigl\|\mathbb{E}_{\rvy_n}\nabla F(\rvy_n)\Bigr\|
\end{align*}
where the last equality follows by probability conversion: the $\rvy_n$ is randomly distributed over $\{\rvx_{t}\}_{t=1}^{n}$ as $\mathbb{P}(\rvy_n=\rvx_t) \;=\; \frac{(1-\beta)\beta^{n-t}}{1-\beta^n}, t=1,\ldots,n$.

\medskip
\noindent\textcircled{2}:\;
\begin{equation*}
\mathbb{E}\sum_{t=1}^{n}\beta^{n-t}\langle \nabla F(\rvx_t) - \rvg_t, -\rvu_n\rangle = \mathbb{E}\sum_{t=1}^{n}\beta^{n-t}\langle \boldsymbol{\xi}_t,\rvu_n\rangle \le 2(1 - \beta)^{-\frac{1}{p}}\sigma D,
\end{equation*}
where the last inequality uses the result $\mathbb{E}\sum_{s=1}^{t}\beta^{t-s}\langle\boldsymbol{\xi}_{s}, \rvu\rangle \le 2(1 - \beta)^{-\frac{1}{p}}\sigma D$ in Lemma~\ref{lemma_discounted_sum_iid};

\medskip
\noindent\textcircled{3}:\;
\begin{equation*}
\mathbb{E}\sum_{t=1}^{n}\beta^{n-t}\langle \rvg_t, \rvz_t-\rvu_n\rangle
=
\mathbb{E}\mathrm{Regret}^{[\beta]}_n(\rvu_n).
\end{equation*}

\medskip
Combining the results in above three steps gives
\begin{align}
\mathbb{E}\left[\sum_{t=1}^{n}\beta^{n-t}\bigl(F(\rvx_t)-F(\rvx_{t-1})\bigr)\right]\le
-D\frac{1 - \beta^{n}}{1 - \beta}\mathbb{E}\Bigl\|\mathbb{E}_{\rvy_n}\nabla F(\rvy_n)\Bigr\| + 2(1 - \beta)^{-\frac{1}{p}}\sigma D + \mathbb{E}\mathrm{Regret}^{[\beta]}_n(\rvu_n),
\end{align}
which completes the proof.
\end{proof}


\LemmaHeavyTailConversion*
\begin{proof}
We start from a change of summation and telescoping,
\begin{align*}
\sum_{n=1}^{T}\sum_{t=1}^{n}(1-\beta)\beta^{n-t}\bigl(F(\rvx_t)-F(\rvx_{t-1})\bigr)
&=
F(\rvx_T)-F(\rvx_0)-\sum_{t=1}^{T}\beta^{T-t+1}\bigl(F(\rvx_t)-F(\rvx_{t-1})\bigr).
\end{align*}
Rearranging and using $F(\rvx_0)-F(\rvx_T)\le \Delta$ gives
\begin{align}
-\Delta \le \mathbb{E}\Bigg[
\underbrace{\sum_{n=1}^{T}\sum_{t=1}^{n}(1-\beta)\beta^{n-t}\bigl(F(\rvx_t)-F(\rvx_{t-1})\bigr)}_{A}
\Bigg]
+
\mathbb{E}\Bigg[
\underbrace{\sum_{t=1}^{T}\beta^{T-t+1}\bigl(F(\rvx_t)-F(\rvx_{t-1})\bigr)}_{B}
\Bigg].
\label{eq:AB_start}
\end{align}

Using the result in Lemma~\ref{lem_key_single_in_conversion}, we have that
\begin{align*}
-\Delta &\le (1-\beta)\sum_{t=1}^{T}\left(-D\frac{1 - \beta^{t}}{1 - \beta}\mathbb{E}\Bigl\|\mathbb{E}_{\rvy_t}\nabla F(\rvy_t)\Bigr\| + 2(1 - \beta)^{-\frac{1}{p}}\sigma D + \mathbb{E}\mathrm{Regret}^{[\beta]}_t(\rvu_t),\right) \\
&\qquad + \beta\left(-D\frac{1 - \beta^{T}}{1 - \beta}\mathbb{E}\Bigl\|\mathbb{E}_{\rvy_T}\nabla F(\rvy_T)\Bigr\| + 2(1 - \beta)^{-\frac{1}{p}}\sigma D + \mathbb{E}\mathrm{Regret}^{[\beta]}_{T}(\rvu_T),\right)
\end{align*}

Rearranging, we obtain
\begin{align*}
&D\sum_{t=1}^{T}(1 - \beta^{t})\mathbb{E}\Bigl\|\mathbb{E}_{\rvy_t}\nabla F(\rvy_t)\Bigr\| + D\frac{\beta - \beta^{T+1}}{1 - \beta}\mathbb{E}\left\|\mathbb{E}_{\rvy_T}\nabla F(\rvy_T)\right\|  \\
& \qquad\le (1-\beta)\sum_{t=1}^{T}\mathbb{E}\left[\mathrm{Regret}^{[\beta]}_t(\rvu_t)\right] + \beta\mathbb{E}\left[\mathrm{Regret}^{[\beta]}_T(\rvu_T)\right] + 2((1-\beta)T + \beta)(1 - \beta)^{-\frac{1}{p}}\sigma D + \Delta 
\end{align*}
By probability conversion over left hand side: since $\sum_{t=1}^{T}(1 - \beta^{t}) + \frac{\beta - \beta^{T+1}}{1 - \beta} = T$, $\tau$ is a random index among $[T]$ such that $\mathbb{P}(\tau = t) = \begin{cases}
\frac{1-\beta^{t}}{T} & \text{if}\quad t = 1,\cdots,T-1,\\
\frac{1 - \beta^{T}}{(1-\beta)T} & \text{if} \quad t = T
\end{cases}$. And, expectations are written as $\mathbb{E}_{\rvy_{\tau}}[\cdot]$ since we condition on the trajectory for the conversion analysis, and the only remaining randomness is the sampling of $\rvy_{\tau}$. 
It suffices to have
\begin{align*} &DT\mathbb{E}_{\tau}\left\|\mathbb{E}_{\rvy_{\tau}}\nabla F(\rvy_{\tau})\right\|  \le (1-\beta)\sum_{t=1}^{T}\mathbb{E}\left[\mathrm{Regret}^{[\beta]}_t(\rvu_t)\right] + \beta\mathbb{E}\left[\mathrm{Regret}^{[\beta]}_T(\rvu_T)\right] + 2((1-\beta)T + \beta)(1 - \beta)^{-\frac{1}{p}}\sigma D + \Delta \\
&\mathbb{E}_{\tau}\left\|\mathbb{E}_{\rvy_{\tau}}\nabla F(\rvy_{\tau})\right\| \le \frac{1}{DT}\left((1-\beta)\sum_{t=1}^{T}\mathbb{E}\left[\mathrm{Regret}^{[\beta]}_t(\rvu_t)\right] + \beta\mathbb{E}\left[\mathrm{Regret}^{[\beta]}_T(\rvu_T)\right]\right) + 2(1-\beta + \frac{\beta}{T})(1 - \beta)^{-\frac{1}{p}}\sigma + \frac{\Delta}{DT}.
\end{align*}
This concludes the proof.
\end{proof}

\section{Missing Proof of Section~\ref{sec_alg}}
\label{append_sec_alg}

\PropNonincreasingStepSize*
\begin{proof}
Let $c_{s}: = \frac{(1 - \beta_{1})\beta_{1}^{s-1}}{(1 - \beta_{1}^{s-1})(\sqrt{b_{s}} + \epsilon)}, \quad s\ge 2$, and set $c_{1}: = c_{2}$. Since $\alpha_{s} = \eta c_{s}$, it suffices to prove $c_{s} \le c_{s-1}$ for $s \ge 3$. The case $s = 2$ holds by initialization.

Fix $s\ge 3$ and set $k = s- 2$. By bias correction,
\[
b_{s} = w_{s}b_{s-1} + (1 - w_{s})\|\rvv_{s-1}\|^{2}, \quad w_{s}=\frac{\beta_{2}(1 - \beta_{2}^{k})}{1 - \beta_{2}^{k+1}},
\]
hence $b_{s} \ge w_{s}b_{s-1}$. Define $\phi_{k}(\beta):=\frac{\beta(1 - \beta^{k})}{1 - \beta^{k+1}}$, so that $w_{s} = \phi_{k}(\beta_{2})$.

Using $\frac{x + \epsilon}{\sqrt{w_{s}}x + \epsilon} \le \frac{1}{\sqrt{w_{s}}}$ for $x \ge 0$,
\[
\frac{c_{s}}{c_{s-1}} \le \frac{\beta_{1}(1 - \beta_{1}^{k})}{1 - \beta_{1}^{k+1}}\cdot\frac{\sqrt{b_{s-1}} + \epsilon}{\sqrt{w_{s}b_{s-1}} + \epsilon} \le \frac{\phi_{k}(\beta_{1})}{\sqrt{\phi_{k}(\beta_{2})}}.
\]

It remains to show $\phi_{k}(\beta_{1})^{2} \le \phi_{k}(\beta_{2})$. Since $\phi_{k}(\beta) = \frac{\sum_{j=1}^{k}\beta^{j}}{\sum_{j=0}^{k}\beta^{j}} = 1 - \frac{1}{\sum_{j=0}^{k}\beta^{j}}$, we have $0 < \phi_{k}(\beta) < 1$, and $\phi_{k}$ is increasing on $0, 1$. Thus, using $\beta_{1} \le \beta_{2}$, $\phi_{k}(\beta_{1})^{2} \le \phi_{k}(\beta_{1}) \le \phi_{k}(\beta_{2})$.

hence $c_{s}/c_{s-1}\le 1$, so $c_{s} \le c_{s-1}$ for all $s\ge 2$. Therefore $\alpha_{s} = \eta c_{s}$ is non-increasing.
\end{proof}

\PropStaticIncrementalBound*
\begin{proof}
Recall that
\[
\rvz_t
=
-\eta_t
\frac{\sum_{s=1}^{t-1} w_{1,t-1,s}\rvv_s}
{\sqrt{\sum_{s=1}^{t-1} w_{2,t-1,s}\|\rvv_s\|^2} + \epsilon},
\]
where
\[
w_{1,t,s} := \frac{(1-\beta_1)\beta_1^{t-s}}{1-\beta_1^t},
\quad
w_{2,t,s} := \frac{(1-\beta_2)\beta_2^{t-s}}{1-\beta_2^t}.
\]

Then discarding the $\epsilon$ term, we obtain
\begin{align*}
\|\rvz_t\|^{2}
&\le
\eta_t^2
\frac{\bigl\|\sum_{s=1}^{t-1} w_{1,t-1,s}\rvv_s\bigr\|^2}
{\sum_{s=1}^{t-1} w_{2,t-1,s}\|\rvv_s\|^2} 
\le
\eta_t^2
\frac{\sum_{s=1}^{t-1} w_{1,t-1,s}\|\rvv_s\|^2}
{\sum_{s=1}^{t-1} w_{2,t-1,s}\|\rvv_s\|^2},
\end{align*}
where the inequality follows from Jensen’s inequality since $\sum_s w_{1,t-1,s}=1$.

Rewriting the numerator gives
\[
\|\rvz_t\|^2
\le
\eta_t^2
\max_{1\le s\le t-1}
\frac{w_{1,t-1,s}}{w_{2,t-1,s}}.
\]
A direct computation yields
\begin{align*}
\max_{1\le s\le t-1}
\frac{w_{1,t-1,s}}{w_{2,t-1,s}}
&=
\frac{1-\beta_1}{1-\beta_2}
\frac{1-\beta_2^{t-1}}{1-\beta_1^{t-1}}
\max_{1\le s\le t-1}
\Bigl(\frac{\beta_1}{\beta_2}\Bigr)^{t-1-s}\le
\frac{1-\beta_1}{1-\beta_2},
\end{align*}

where we used $\beta_1\le \beta_2$.
Finally, Remark~\ref{remark_beta2_Cbeta} ensures
\(
\frac{1-\beta_1}{1-\beta_2}
\le
C(\beta).
\)

It suffice to have $\|\rvz_t\|^2 \le \eta_{t}^{2}C(\beta)$.
Taking square roots completes the proof.
\end{proof}

\PropHeavyTailControl*
\begin{proof}
Recall that
\[
b_{t+1}
=
\sum_{s=1}^t w_{s,t}\|\rvv_s\|^2,
\qquad
w_{s,t}
=
\frac{1-\beta_2}{1-\beta_2^t}\beta_2^{t-s}.
\]
Since $\rvv_s=\boldsymbol{\mu}_s+\boldsymbol{\xi}_s$, we have
\[
\sqrt{b_{t+1}}
=
\left(
\sum_{s=1}^t w_{s,t}\|\boldsymbol{\mu}_s+\boldsymbol{\xi}_s\|^2
\right)^{1/2}.
\]
By the triangle inequality in the product Hilbert space,
\[
\sqrt{b_{t+1}}
\le
\left(
\sum_{s=1}^t w_{s,t}\|\boldsymbol{\mu}_s\|^2
\right)^{1/2}
+
\left(
\sum_{s=1}^t w_{s,t}\|\boldsymbol{\xi}_s\|^2
\right)^{1/2}.
\]
Since $\|\boldsymbol{\mu}_s\|\le G$ and $\sum_s w_{s,t}=1$,
$
\left(
\sum_{s=1}^t w_{s,t}\|\boldsymbol{\mu}_s\|^2
\right)^{1/2}
\le G.
$
Therefore,
\[
\mathbb{E}\sqrt{b_{t+1}}
\le
G
+
\mathbb{E}
\left(
\sum_{s=1}^t w_{s,t}\|\boldsymbol{\xi}_s\|^2
\right)^{1/2}.
\]
By Lyapunov's inequality,
\[
\mathbb{E}
\left(
\sum_{s=1}^t w_{s,t}\|\boldsymbol{\xi}_s\|^2
\right)^{1/2}
\le
\left[
\mathbb{E}
\left(
\sum_{s=1}^t w_{s,t}\|\boldsymbol{\xi}_s\|^2
\right)^{p/2}
\right]^{1/p}.
\]
Because $p/2\le 1$, the map $x\mapsto x^{p/2}$ is subadditive on $\mathbb{R}_+$, and hence
\[
\left(
\sum_{s=1}^t w_{s,t}\|\boldsymbol{\xi}_s\|^2
\right)^{p/2}
\le
\sum_{s=1}^t w_{s,t}^{p/2}\|\boldsymbol{\xi}_s\|^p.
\]
Taking expectations and using
$\mathbb{E}\|\boldsymbol{\xi}_s\|^p\le \sigma^p$ gives
\[
\mathbb{E}
\left(
\sum_{s=1}^t w_{s,t}\|\boldsymbol{\xi}_s\|^2
\right)^{p/2}
\le
\sigma^p
\sum_{s=1}^t w_{s,t}^{p/2}.
\]
Thus
\[
\mathbb{E}\sqrt{b_{t+1}}
\le
G
+
\sigma
\left(
\sum_{s=1}^t w_{s,t}^{p/2}
\right)^{1/p}.
\]

It remains to compute the weight sum. We have
\[
\sum_{s=1}^t w_{s,t}^{p/2}
=
\left(
\frac{1-\beta_2}{1-\beta_2^t}
\right)^{p/2}
\sum_{s=1}^t \beta_2^{(t-s)p/2}
=
\left(
\frac{1-\beta_2}{1-\beta_2^t}
\right)^{p/2}
\frac{1-\beta_2^{tp/2}}{1-\beta_2^{p/2}} .
\]
Therefore,
\[
\mathbb{E}\sqrt{b_{t+1}}
\le
G
+
\sigma
\left[
\left(
\frac{1-\beta_2}{1-\beta_2^t}
\right)^{p/2}
\frac{1-\beta_2^{tp/2}}{1-\beta_2^{p/2}}
\right]^{1/p}.
\]
Since for $q=p/2\in(0,1]$ and $x\in[0,1]$,
$
1-x^q\le (1-x)^q,
$
we get
$
\frac{1-\beta_2^{tp/2}}{(1-\beta_2^t)^{p/2}}
\le 1.
$
Hence
\[
\left(
\sum_{s=1}^t w_{s,t}^{p/2}
\right)^{1/p}
\le
\frac{(1-\beta_2)^{1/2}}
{(1-\beta_2^{p/2})^{1/p}}.
\]
Finally, by concavity of $x^{p/2}$,
$
1-\beta_2^{p/2}
\ge
\frac{p}{2}(1-\beta_2),
$
so
\[
\frac{(1-\beta_2)^{1/2}}
{(1-\beta_2^{p/2})^{1/p}}
\le
\left(\frac{2}{p}\right)^{1/p}
(1-\beta_2)^{\frac12-\frac1p}.
\]
The claim follows.
\end{proof}

\section{Missing Proofs of Section~\ref{sec_guarantees}}
\label{append_sec_guarantees}

\LemmaMostDifficultOne*
\begin{proof}
The proof consists of three parts.
\paragraph{Part 1.}
We begin with
\begin{align*}
&\frac{\alpha_{s}}{2}\beta_{1}^{-2s}\|\overline{\rvv}_{s}\|^2 \\
& \le \frac{\eta}{2}\frac{1-\beta_{1}}{1-\beta_{1}^{s-1}}
  \frac{\sqrt{1-\beta_{2}^{s-1}}}{\sqrt{1-\beta_{2}}}
  \frac{1}{\sqrt{\beta_{2}^{s}}}
  \frac{\beta_{1}^{-s-1}\|\overline{\rvv}_{s}\|^2}
       {\sqrt{\sum_{n=1}^{s-1}\beta_{2}^{-n-1}\|\rvv_{n}\|^{2}}}\\
& \stackrel{(a)}{\le}
  \frac{\eta}{2}\frac{1-\beta_{1}}{1-\beta_{1}^{s-1}}
  \frac{\sqrt{1-\beta_{2}^{s-1}}}{\sqrt{1-\beta_{2}}}
  \frac{\beta_{1}^{-s-1}}{\beta_{2}^{-s-1}}
  \frac{1}{\sqrt{\beta_{2}^{s}}}
  \frac{\sqrt{2}\,\beta_{2}^{-s-1}\|\overline{\rvv}_{s}\|^2}
       {\sqrt{\beta_{2}^{-s-1}\|\overline{\rvv}_{s}\|^2
         +\sum_{n=1}^{s-1}\beta_{2}^{-n-1}\|\rvv_{n}\|^{2}}}\\
&\stackrel{(b)}{\le}
  \frac{\eta}{2}
  \frac{1-\beta_{1}}{\sqrt{1-\beta_{2}}}
  \frac{\sqrt{1-\beta_{2}^{s-1}}}{1-\beta_{1}^{s-1}}
  \frac{\beta_{1}^{-s-1}}{\beta_{2}^{-s-1}}
  \frac{1}{\sqrt{\beta_{2}^{s}}}
  \cdot 2\sqrt{2}
    \Bigl(\sqrt{\beta_{2}^{-s-1}\|\overline{\rvv}_{s}\|^2
         +\sum_{n=1}^{s-1}\beta_{2}^{-n-1}\|\rvv_{n}\|^{2}}
    -\sqrt{\textstyle\sum_{n=1}^{s-1}\beta_{2}^{-n-1}\|\rvv_{n}\|^{2}}\Bigr) \\
&\le
  \frac{\eta}{2}
  \frac{1-\beta_{1}}{\sqrt{1-\beta_{2}}}
  \frac{\sqrt{1-\beta_{2}^{s-1}}}{1-\beta_{1}^{s-1}}
  \frac{\beta_{1}^{-s-1}}{\beta_{2}^{-s-1}}
  \frac{1}{\sqrt{\beta_{2}^{s}}}
  \cdot 2\sqrt{2}
  \underbrace{
    \Bigl(\sqrt{\textstyle\sum_{n=1}^{s}\beta_{2}^{-n-1}\|\rvv_{n}\|^{2}}
    -\sqrt{\textstyle\sum_{n=1}^{s-1}\beta_{2}^{-n-1}\|\rvv_{n}\|^{2}}\Bigr)
  }_{=\;Q_s\,-\,Q_{s-1}} \\
  &= \underbrace{
    \sqrt{2}\eta\,
    \frac{(1-\beta_{1})\sqrt{1-\beta_{2}^{s-1}}}
         {(1-\beta_{1}^{s-1})\sqrt{1-\beta_{2}}}\,
    \frac{\beta_{2}^{s/2+1}}{\beta_{1}^{s+1}}
  }_{=\;P_s}
  (Q_{s} - Q_{s-1})
\end{align*}
where 
$(a)$ uses
$\beta_{2}^{-s-1}\|\overline{\rvv}_{s}\|^2 \le \sum_{n=1}^{s-1}\beta_{2}^{s-2n-1}\|\rvv_{n}\|^{2}
\le \sum_{n=1}^{s-1}\beta_{2}^{-n-1}\|\rvv_{n}\|^{2}$, which follows from the definition of $\overline{\rvv}_s$;
$(b)$ uses
$\frac{x}{\sqrt{x+y}}\le 2(\sqrt{x+y}-\sqrt{y})$
for $x,y>0$.

From the result above, we have
\begin{align*}
\sum_{s=1}^{t}\frac{\alpha_s}{2}\beta_1^{t-2s}\|\overline{\rvv}_s\|^2  \le \beta_{1}^{t}\left(\sum_{s=1}^{t}P_{s}(Q_{s} - Q_{s-1})\right) = \beta_{1}^{t}\left(P_t Q_t
  +\sum_{s=1}^{t-1}(P_s-P_{s+1})Q_s\right).
\end{align*}

\paragraph{Part 2.}

Next, we show that the remainder $R(t): = \sum_{s=1}^{t-1}(P_s-P_{s+1})Q_s$ is non-positive for $t\ge T_{0}: = 2+\frac{\log(1/(1 - \beta_{1}))}{\log(\sqrt{\beta_{2}}/\beta_{1})}$.

\paragraph{Part 2.1. The sequence \(\{P_s\}_{s\ge2}\) decreases and then increases.}
Extend
$h(s):=\log P_s$ to real $s\ge2$. Let
$\lambda_i:=-\log\beta_i$ and $\tau:=\lambda_1/\lambda_2\ge 1$. Then
\[
h'(s)
=
\frac{\lambda_2}{2}
\left[
(2\tau-1)
+\frac{1}{e^y-1}
-\frac{2\tau}{e^{\tau y}-1}
\right],
\qquad
y:=\lambda_2(s-1).
\]
After multiplying by the positive quantity
$2(e^y-1)(e^{\tau y}-1)/\lambda_2$, the sign of $h'(s)$ is the sign of
\[
G(y)
:=
(2\tau-1)(e^y-1)(e^{\tau y}-1)
+(e^{\tau y}-1)
-2\tau(e^y-1).
\]

To verify the claimed sign pattern, write \(x=e^y\ge1\). A direct
calculation gives
\[
G'(y)=x\Phi(x),
\]
where
\[
\Phi(x)
=
(2\tau-1)(\tau+1)x^\tau
-2\tau(\tau-1)x^{\tau-1}
-(4\tau-1).
\]
Moreover,
\[
\Phi(1)=-\tau<0,\qquad \Phi(x)\to\infty,
\]
and
\[
\Phi'(x)
=
\tau x^{\tau-2}
\left((2\tau-1)(\tau+1)x-2(\tau-1)^2\right)>0
\quad (x\ge1),
\]
because
\[
(2\tau-1)(\tau+1)-2(\tau-1)^2=5\tau-3>0.
\]
Thus \(G'\) has exactly one positive zero. Since \(G(0)=0\),
\(G'(0)<0\), and \(G(y)\to\infty\), it follows that \(G\) has exactly
one positive zero. Hence \(h'\) changes sign once, from negative to
positive.

Therefore \(h\) first decreases and then increases on \([2,\infty)\).
Consequently, there exists an integer \(s_\star\ge2\) such that
\(P_{s+1}\ge P_s\), we have
\[
P_{s+1}<P_s \quad (2\le s<s_\star),
\qquad
P_{s+1}\ge P_s \quad (s\ge s_\star).
\]

\paragraph{Part 2.2.}
Write
\[
\rho:=\frac{\sqrt{\beta_2}}{\beta_1}>1,
\qquad
C_L:=\frac{\sqrt{2}\eta(1-\beta_1)\beta_2}{\beta_1}.
\]
From the definition of $P_s$,
\[
P_s\ge C_L\rho^s,
\qquad
P_2=\frac{C_L\rho^2}{1-\beta_1}.
\]
Hence $t\ge T_0$ implies $P_t\ge P_2$. 
By the monotonicity pattern established above, this gives
$P_j\le P_t$ for all $2\le j\le t$.
Indeed, terms before the turning point are at most \(P_2\), while terms
after the turning point are nondecreasing and hence at most \(P_t\). Therefore,
\begin{align*}
\sum_{s=1}^{t-1}(P_s-P_{s+1})Q_s
&=
-P_tQ_{t-1}
+
\sum_{s=1}^{t-2}P_{s+1}(Q_{s+1}-Q_s) \\
&\le
-P_tQ_{t-1}
+
P_t\sum_{s=1}^{t-2}(Q_{s+1}-Q_s) \\
&=
-P_tQ_1
\le0.
\end{align*}
Thus, for $t\ge T_0$,
\[
\sum_{s=1}^{t}
\frac{\alpha_s}{2}\beta_1^{t-2s}
\|\overline{\rvv}_s\|^2
\le
\beta_1^tP_tQ_t .
\]

\paragraph{Part 3. }

It remains to bound the expectation of the terminal term. Expanding $P_tQ_t$,
\[
\beta_1^tP_tQ_t
=
\frac{\sqrt{2\beta_2}\eta(1-\beta_1)}
{\beta_1\sqrt{1-\beta_2}}
\frac{\sqrt{1-\beta_2^{t-1}}}
{1-\beta_1^{t-1}}
\sqrt{
\sum_{n=1}^{t}\beta_2^{t-n}\|\rvv_n\|^2
}.
\]
By Corollary~\ref{corollary_cost_verctor_tail_norm_bound},
\begin{align*}
\mathbb E\!\left[
\sqrt{
\sum_{n=1}^{t}\beta_2^{t-n}\|\rvv_n\|^2
}
\right]
&\le
\sqrt{\frac{1 - \beta_{2}^{t}}{1 - \beta_{2}}}\left(G + 2\sigma(1 - \beta_{2})^{-\left(\frac1p-\frac12\right)}\right).
\end{align*}
Therefore,
\begin{align*}
\mathbb E\!\left[\sum_{s=1}^{t}
\frac{\alpha_s}{2}\beta_1^{t-2s}
\|\overline{\rvv}_s\|^2\right] \le \mathbb E[\beta_1^tP_tQ_t]
&\le
\frac{\sqrt{2\beta_2}\eta(1-\beta_1)}
{\beta_1(1-\beta_2)}
\frac{
\sqrt{(1-\beta_2^{t-1})(1-\beta_2^t)}
}
{1-\beta_1^{t-1}}
\left(G + 2\sigma(1 - \beta_{2})^{-\left(\frac1p-\frac12\right)}\right)\\
&\le\eta\frac{2\sqrt{\beta_2}(1-\beta_1)}{\beta_1(1-\beta_2)}\left(G + 2\sigma(1 - \beta_{2})^{-\left(\frac1p-\frac12\right)}\right).
\end{align*}
where the last inequality uses (1) $1-\beta_2^t\le 2(1-\beta_2^{t-1})$ for $t\ge2$; (2) $\frac{1-\beta_2^{t-1}}{1-\beta_1^{t-1}}\le1$ for $\beta_1 \le \beta_2$.

This concludes the proof. 
\end{proof}

\TheoremFinalConvergence*
\begin{proof}
Start from the conversion bound~\eqref{eq_stationary_point_guarantee}
(with $\beta=\beta_1$):
\begin{align*}
\mathbb{E}_{\tau}\Vert\nabla F(\widetilde{\rvx}_{\tau})\Vert^{[\rho]}
&\le
\frac{1}{DT}
\left(
\beta_1\mathbb{E}\left[\mathrm{Regret}^{[\beta]}_T(\rvu_T)\right]
+
(1-\beta_1)\sum_{t=1}^{T}
\mathbb{E}\left[\mathrm{Regret}^{[\beta]}_t(\rvu_t)\right]
\right)
\nonumber
\\
&\qquad
+
2\left(1-\beta_1+\frac{\beta_1}{T}\right)
(1-\beta_1)^{-\frac1p}\sigma
+
\frac{\Delta}{DT}
+
\frac{2\rho\beta_1\sup_{t\in[T]}\mathbb{E}\left[\|\rvz_t\|^2\right]}
{(1-\beta_1)^2}.
\end{align*}
Let
\[
T_0
:=
2+\frac{\log(1/(1-\beta_1))}
{\log(\sqrt{\beta_2}/\beta_1)} ,
\]
where \(T_0\) is rounded up if necessary. We assume
\(\beta_1<\sqrt{\beta_2}\), so that the denominator is positive.

For the first \(T_0\) rounds, we use a crude bound inferred directly
from the regret definition. Since
\[
\|\rvz_s\|\le \frac{\sqrt{C(\beta)}D}{(1-\beta_1)^{1/4}},
\qquad
\|\rvu_t\|\le D,
\]
we have, for any \(t<T_0\),
\begin{align*}
\mathbb{E}\left[\mathrm{Regret}^{[\beta]}_t(\rvu_t)\right]
&=
\mathbb{E}\left[
\sum_{s=1}^{t}
\beta_1^{t-s}
\langle \rvv_s,\rvz_s-\rvu_t\rangle
\right]
\\
&\le
D\left(1+\sqrt{C(\beta)}(1-\beta_1)^{-1/4}\right)
\mathbb{E}\left[
\sum_{s=1}^{t}\beta_1^{t-s}\|\rvv_s\|
\right].
\end{align*}
By Assumption~\ref{ass_vector_assumption},
\[
\mathbb{E}\|\rvv_s\|
\le
\|\boldsymbol{\mu}_s\|
+
\mathbb{E}\|\boldsymbol{\xi}_s\|
\le
G+
\left(\mathbb{E}\|\boldsymbol{\xi}_s\|^p\right)^{1/p}
\le
G+\sigma .
\]
Therefore, summing the first \(T_0\) rounds inside the conversion bound gives
\begin{align*}
&\frac{1-\beta_1}{DT}
\sum_{t<T_0}
\mathbb{E}\left[\mathrm{Regret}^{[\beta]}_t(\rvu_t)\right]
\\
&\qquad\le
\frac{1-\beta_1}{T}
\left(1+\sqrt{C(\beta)}(1-\beta_1)^{-1/4}\right)
\sum_{t<T_0}\sum_{s=1}^{t}
\beta_1^{t-s}\mathbb{E}\|\rvv_s\|
\\
&\qquad=
\frac{1}{T}
\left(1+\sqrt{C(\beta)}(1-\beta_1)^{-1/4}\right)
\sum_{s<T_0}
\left((1-\beta_1)\sum_{t=s}^{T_0-1}\beta_1^{t-s}\right)
\mathbb{E}\|\rvv_s\|
\\
&\qquad\le
\frac{T_0}{T}
\left(1+\sqrt{C(\beta)}(1-\beta_1)^{-1/4}\right)
(G+\sigma).
\end{align*}
For \(t\ge T_0\), Theorem~\ref{theorem_regret_bound} gives
\[
\mathbb{E}\left[\mathrm{Regret}^{[\beta]}_t(\rvu_t)\right]
\le
\frac{
DC(R)\left(G+2\sigma(1-\beta_2)^{-\left(\frac1p-\frac12\right)}\right)
}
{(1-\beta_1)^{3/4}} .
\]
Thus, for \(T\ge T_0\),
\begin{align*}
&\frac{1}{DT}
\left(
\beta_1\mathbb{E}\left[\mathrm{Regret}^{[\beta]}_T(\rvu_T)\right]
+
(1-\beta_1)\sum_{t=1}^{T}
\mathbb{E}\left[\mathrm{Regret}^{[\beta]}_t(\rvu_t)\right]
\right)
\\
&\qquad\le
\left(1-\beta_1+\frac{\beta_1}{T}\right)
\frac{
C(R)\left(G+2\sigma(1-\beta_2)^{-\left(\frac1p-\frac12\right)}\right)
}
{(1-\beta_1)^{3/4}}
\\
&\qquad\quad
+
\frac{T_0}{T}
\left(1+\sqrt{C(\beta)}(1-\beta_1)^{-1/4}\right)
(G+\sigma).
\end{align*}

Substituting the increment bound from Proposition~\ref{prop_staticincremental_bound} and $\eta = \frac{D}{(1 - \beta_{1})^{1/4}}$,
\[
\|\rvz_t\|
\le \sqrt{C(\beta)}\eta
=
\frac{\sqrt{C(\beta)}D}{(1-\beta_1)^{1/4}},
\]
and choosing
\(
D=
\frac{(1-\beta_1)^{5/4}\epsilon^{1/2}}
{C(\beta)\rho^{1/2}},
\)
we obtain
\begin{align*}
\mathbb{E}_{\tau}\Vert\nabla F(\widetilde{\rvx}_{\tau})\Vert^{[\rho]}
&\le
\left(1-\beta_1+\frac{\beta_1}{T}\right)
\frac{
C(R)\left(G+2\sigma(1-\beta_2)^{-\left(\frac1p-\frac12\right)}\right)
}
{(1-\beta_1)^{3/4}}
\\
&\qquad
+
\frac{T_0}{T}
\left(1+\sqrt{C(\beta)}(1-\beta_1)^{-1/4}\right)
(G+\sigma)
\\
&\qquad
+
2\left(1-\beta_1+\frac{\beta_1}{T}\right)
(1-\beta_1)^{-1/p}\sigma
+
\frac{\sqrt{C(\beta)}\rho^{1/2}\Delta}
{(1-\beta_1)^{5/4}\epsilon^{1/2}T}
+
2\epsilon .
\end{align*}
Expanding the first and third terms gives
\begin{align*}
\mathbb{E}_{\tau}\Vert\nabla F(\widetilde{\rvx}_{\tau})\Vert^{[\rho]}
&\le
C(R)\left(G+2\sigma(1-\beta_2)^{-\left(\frac1p-\frac12\right)}\right)
(1-\beta_1)^{1/4}
+
2\sigma(1-\beta_1)^{1-\frac1p}
\\
&\qquad
+
\frac{
C(R)\beta_1
\left(G+2\sigma(1-\beta_2)^{-\left(\frac1p-\frac12\right)}\right)
}
{(1-\beta_1)^{3/4}T}
+
\frac{2\beta_1\sigma}
{(1-\beta_1)^{1/p}T}
\\
&\qquad
+
\frac{\sqrt{C(\beta)}\rho^{1/2}\Delta}
{(1-\beta_1)^{5/4}\epsilon^{1/2}T}
+
\frac{T_0}{T}
\left(1+\sqrt{C(\beta)}(1-\beta_1)^{-1/4}\right)
(G+\sigma)
+
2\epsilon .
\end{align*}
Hence, using
\(
1-\beta_1\le C(\beta)(1-\beta_2),
\)
we have
\begin{align*}
\mathbb{E}_{\tau}\Vert\nabla F(\widetilde{\rvx}_{\tau})\Vert^{[\rho]}
&\le
C(R)G(1-\beta_1)^{1/4}
+
2C(R)C(\beta)^{\frac1p-\frac12}
\sigma(1-\beta_1)^{\frac34-\frac1p}
+
2\sigma(1-\beta_1)^{1-\frac1p}
\\
&\qquad
+
\frac{C(R)\beta_1G}
{(1-\beta_1)^{3/4}T}
+
\frac{
2C(R)C(\beta)^{\frac1p-\frac12}\beta_1\sigma
}
{(1-\beta_1)^{\frac14+\frac1p}T}
+
\frac{2\beta_1\sigma}
{(1-\beta_1)^{1/p}T}
\\
&\qquad
+
\frac{\sqrt{C(\beta)}\rho^{1/2}\Delta}
{(1-\beta_1)^{5/4}\epsilon^{1/2}T}
+
\frac{T_0}{T}
\left(1+\sqrt{C(\beta)}(1-\beta_1)^{-1/4}\right)
(G+\sigma)
+
2\epsilon .
\end{align*}

Now assume \(p>\frac43\), \(0<\epsilon\le C(R)G+\sigma\), and
\(\beta_1 \le \beta_{2}<\sqrt{\beta_2}\). Choose
\[
\beta_1
=
1-
\left(
\frac{\epsilon}{G+\sigma}
\right)^{\frac{4p}{3p-4}} .
\]
Then
\(
(1-\beta_1)^{1/4}
\le
\frac{\epsilon}{G+\sigma}
\)
,
\(
(1-\beta_1)^{\frac34-\frac1p}
=
\frac{\epsilon}{G+\sigma}
\)
, and 
\(
(1-\beta_1)^{1-\frac1p}
\le
\frac{\epsilon}{G+\sigma}.
\)
Thus, the first three non-\(T^{-1}\) terms are bounded by
\(
C(R)\epsilon
+
2C(R)C(\beta)^{\frac1p-\frac12}\epsilon
+
2\epsilon .
\)
Therefore,
\begin{align*}
\mathbb{E}_{\tau}\Vert\nabla F(\widetilde{\rvx}_{\tau})\Vert^{[\rho]}
&\le
\frac{
C(R)\beta_1G
\left(C(R)G+\sigma\right)^{\frac{3p}{3p-4}}
}
{
\epsilon^{\frac{3p}{3p-4}}T
}
+
\frac{
2C(R)C(\beta)^{\frac1p-\frac12}\beta_1\sigma
\left(C(R)G+\sigma\right)^{\frac{p+4}{3p-4}}
}
{
\epsilon^{\frac{p+4}{3p-4}}T
}
\\
&\qquad
+
\frac{
2\beta_1\sigma
\left(C(R)G+\sigma\right)^{\frac{4}{3p-4}}
}
{
\epsilon^{\frac{4}{3p-4}}T
}
+
\frac{
\sqrt{C(\beta)}\rho^{1/2}\Delta
\left(C(R)G+\sigma\right)^{\frac{5p}{3p-4}}
}
{
\epsilon^{\frac{5p}{3p-4}+\frac12}T
}
\\
&\qquad
+
\frac{(G+\sigma)T_{0}}{T}\cdot
\left(
1+
\sqrt{C(\beta)}
\left(C(R)G+\sigma\right)^{\frac{p}{3p-4}}
\epsilon^{-\frac{p}{3p-4}}
\right)
\\
&\qquad
+ C(R)\epsilon
+
2C(R)C(\beta)^{\frac1p-\frac12}\epsilon
+
4\epsilon .
\end{align*}
It is sufficient to choose
\begin{align*}
T
\ge
\max\Bigg\{
&
2\left(1+C(R)C(\beta)^{\frac1p-\frac12}\right)
\left(C(R)G+\sigma\right)^{\frac{4p}{3p-4}}
\epsilon^{-\frac{4p}{3p-4}},
\\
&
\sqrt{C(\beta)}\rho^{1/2}\Delta
\left(C(R)G+\sigma\right)^{\frac{5p}{3p-4}}
\epsilon^{-\left(\frac{5p}{3p-4}+\frac32\right)},
\\
&
(G+\sigma)\epsilon^{-1}T_{0} + 
\sqrt{C(\beta)}
\left(C(R)G+\sigma\right)^{\frac{4p-4}{3p-4}}
\epsilon^{-\frac{4p-4}{3p-4}}T_{0}
\Bigg\},
\end{align*}
where the last term can be further reformulated as $\mathcal{O}\left(\left(C(R)G+\sigma\right)^{\frac{8p- 4}{3p-4}}
\epsilon^{-\frac{8p -4}{3p-4}}\log\frac{G + \sigma}{\epsilon}\right)$
due to $T_0
:=
2+\frac{\log(1/(1-\beta_1))}
{\log(\sqrt{\beta_2}/\beta_1)} = \mathcal{O}\left(\left(\frac{G+\sigma}{\epsilon}\right)^{\frac{4p}{3p-4}}\log\frac{G + \sigma}{\epsilon}\right)$.

The third candidate ensures both \(T\ge T_0\) and that the burn-in term is
at most \(\epsilon\). The first candidate controls the first three
\(T^{-1}\) terms by at most \(2\epsilon\), and the second candidate controls
the \(\Delta\)-term by at most \(\epsilon\). 
Consequently,
\[
\mathbb{E}_{\tau}\Vert\nabla F(\widetilde{\rvx}_{\tau})\Vert^{[\rho]}
\le
\left(
8+
C(R)+
2C(R)C(\beta)^{\frac1p-\frac12}
\right)\epsilon .
\]

Meanwhile, we observe that the second term in the selection of T dominates under fixed constants and nondegenerate $\Delta, \rho$ asymptotically.
This proves the desired stationarity guarantee for \(p>\frac43\).
\end{proof}

\CorollaryWithKnowledgeD*
\begin{proof}
Start from the unsimplified conversion bound~\eqref{eq_stationary_point_guarantee}
(with $\beta=\beta_1$):
\begin{align*}
\mathbb{E}_{\tau}\Vert\nabla F(\widetilde{\rvx}_{\tau})\Vert^{[\rho]}
&\le
\frac{1}{DT}
\left(
\beta_1\mathbb{E}\left[\mathrm{Regret}^{[\beta]}_T(\rvu_T)\right]
+
(1-\beta_1)\sum_{t=1}^{T}
\mathbb{E}\left[\mathrm{Regret}^{[\beta]}_t(\rvu_t)\right]
\right)
\nonumber
\\
&\qquad
+
2\left(1-\beta_1+\frac{\beta_1}{T}\right)
(1-\beta_1)^{-\frac1p}\sigma
+
\frac{\Delta}{DT}
+
\frac{2\rho\beta_1\sup_{t\in[T]}\mathbb{E}\left[\|\rvz_t\|^2\right]}
{(1-\beta_1)^2}.
\end{align*}
Recall
\(
T_0
:=
2+\frac{\log(1/(1-\beta_1))}
{\log(\sqrt{\beta_2}/\beta_1)} ,
\)
where \(T_0\) is rounded up if necessary. 

For the first \(T_0\) rounds, we use a crude bound inferred directly
from the regret definition. Since
\[
\|\rvz_s\|\le D,
\qquad
\|\rvu_t\|\le D,
\]
we have, for any \(t<T_0\),
\begin{align*}
\mathbb{E}\left[\mathrm{Regret}^{[\beta]}_t(\rvu_t)\right]
&=
\mathbb{E}\left[
\sum_{s=1}^{t}
\beta_1^{t-s}
\langle \rvv_s,\rvz_s-\rvu_t\rangle
\right]\le
2D
\mathbb{E}\left[
\sum_{s=1}^{t}\beta_1^{t-s}\|\rvv_s\|
\right].
\end{align*}
By Assumption~\ref{ass_vector_assumption},
\[
\mathbb{E}\|\rvv_s\|
\le
\|\boldsymbol{\mu}_s\|
+
\mathbb{E}\|\boldsymbol{\xi}_s\|
\le
G+
\left(\mathbb{E}\|\boldsymbol{\xi}_s\|^p\right)^{1/p}
\le
G+\sigma .
\]
Therefore, summing the first \(T_0\) rounds inside the conversion bound gives
\begin{align*}
\frac{1-\beta_1}{DT}
\sum_{t<T_0}
\mathbb{E}\left[\mathrm{Regret}^{[\beta]}_t(\rvu_t)\right]
&\le
\frac{2(1-\beta_1)}{T}
\sum_{t<T_0}\sum_{s=1}^{t}
\beta_1^{t-s}\mathbb{E}\|\rvv_s\|
\\
&=
\frac{2}{T}
\sum_{s<T_0}
\left((1-\beta_1)\sum_{t=s}^{T_0-1}\beta_1^{t-s}\right)
\mathbb{E}\|\rvv_s\|
\\
&\le
\frac{2T_0}{T}(G+\sigma).
\end{align*}
For \(t\ge T_0\), Corollary~\ref{corollary_inc_norm_bound} gives
\[
\mathbb{E}\left[\mathrm{Regret}^{[\beta]}_t(\rvu_t)\right]
\le
\frac{
DC(R)\left(G+2\sigma(1-\beta_2)^{-\left(\frac1p-\frac12\right)}\right)
}
{(1-\beta_1)^{1/2}} .
\]
Thus, for \(T\ge T_0\),
\begin{align*}
&\frac{1}{DT}
\left(
\beta_1\mathbb{E}\left[\mathrm{Regret}^{[\beta]}_T(\rvu_T)\right]
+
(1-\beta_1)\sum_{t=1}^{T}
\mathbb{E}\left[\mathrm{Regret}^{[\beta]}_t(\rvu_t)\right]
\right)
\\
&\qquad\le
\left(1-\beta_1+\frac{\beta_1}{T}\right)
\frac{
C(R)\left(G+2\sigma(1-\beta_2)^{-\left(\frac1p-\frac12\right)}\right)
}
{(1-\beta_1)^{1/2}}
+
\frac{2T_0}{T}(G+\sigma).
\end{align*}

Substituting the clipped increment bound \(\|\rvz_t\|\le D\), and choosing
\(
D=
\frac{(1-\beta_1)\epsilon^{1/2}}
{\rho^{1/2}},
\)
we obtain
\begin{align*}
\mathbb{E}_{\tau}\Vert\nabla F(\widetilde{\rvx}_{\tau})\Vert^{[\rho]}
&\le
\left(1-\beta_1+\frac{\beta_1}{T}\right)
\frac{
C(R)\left(G+2\sigma(1-\beta_2)^{-\left(\frac1p-\frac12\right)}\right)
}
{(1-\beta_1)^{1/2}}
\\
&\qquad
+
\frac{2T_0}{T}(G+\sigma)
+
2\left(1-\beta_1+\frac{\beta_1}{T}\right)
(1-\beta_1)^{-1/p}\sigma
\\
&\qquad
+
\frac{\rho^{1/2}\Delta}
{(1-\beta_1)\epsilon^{1/2}T}
+
2\epsilon .
\end{align*}
Expanding the first and third terms gives
\begin{align*}
\mathbb{E}_{\tau}\Vert\nabla F(\widetilde{\rvx}_{\tau})\Vert^{[\rho]}
&\le
C(R)\left(G+2\sigma(1-\beta_2)^{-\left(\frac1p-\frac12\right)}\right)
(1-\beta_1)^{1/2}
+
2\sigma(1-\beta_1)^{1-\frac1p}
\\
&\qquad
+
\frac{
C(R)\beta_1
\left(G+2\sigma(1-\beta_2)^{-\left(\frac1p-\frac12\right)}\right)
}
{(1-\beta_1)^{1/2}T}
+
\frac{2\beta_1\sigma}
{(1-\beta_1)^{1/p}T}
\\
&\qquad
+
\frac{\rho^{1/2}\Delta}
{(1-\beta_1)\epsilon^{1/2}T}
+
\frac{2T_0}{T}(G+\sigma)
+
2\epsilon .
\end{align*}
Using
\(
1-\beta_1\le C(\beta)(1-\beta_2),
\)
we obtain
\begin{align*}
\mathbb{E}_{\tau}\Vert\nabla F(\widetilde{\rvx}_{\tau})\Vert^{[\rho]}
&\le
C(R)G(1-\beta_1)^{1/2}
+
2C(R)C(\beta)^{\frac1p-\frac12}
\sigma(1-\beta_1)^{1-\frac1p}
+
2\sigma(1-\beta_1)^{1-\frac1p}
\\
&\qquad
+
\frac{C(R)\beta_1G}
{(1-\beta_1)^{1/2}T}
+
\frac{
2\left(1+C(R)C(\beta)^{\frac1p-\frac12}\right)\beta_1\sigma
}
{(1-\beta_1)^{1/p}T}
\\
&\qquad
+
\frac{\rho^{1/2}\Delta}
{(1-\beta_1)\epsilon^{1/2}T}
+
\frac{2T_0}{T}(G+\sigma)
+
2\epsilon .
\end{align*}

Now assume \(0<\epsilon\le G+\sigma\), and choose
\(
\beta_1
=
1-
\left(
\frac{\epsilon}{G+\sigma}
\right)^{\frac{p}{p-1}} .
\)
Then
\(
(1-\beta_1)^{1-\frac1p}
=
\frac{\epsilon}{G+\sigma}
\)
and 
\(
(1-\beta_1)^{1/2}
\le
\frac{\epsilon}{G+\sigma}.
\)
Thus, the first three non-\(T^{-1}\) terms are bounded by
\(
C(R)\epsilon+
2C(R)C(\beta)^{\frac1p-\frac12}\epsilon
+
2\epsilon .
\)
Therefore,
\begin{align*}
\mathbb{E}_{\tau}\Vert\nabla F(\widetilde{\rvx}_{\tau})\Vert^{[\rho]}
&\le
\frac{
C(R)\beta_1G
\left(C(R)G+\sigma\right)^{\frac{p}{2p-2}}
}
{
\epsilon^{\frac{p}{2p-2}}T
}
+
\frac{
2\left(1+C(R)C(\beta)^{\frac1p-\frac12}\right)\beta_1\sigma
\left(C(R)G+\sigma\right)^{\frac{1}{p-1}}
}
{
\epsilon^{\frac{1}{p-1}}T
}
\\
&\qquad
+
\frac{
\rho^{1/2}\Delta
\left(C(R)G+\sigma\right)^{\frac{p}{p-1}}
}
{
\epsilon^{\frac{p}{p-1}+\frac12}T
}
+
\frac{2(G+\sigma)T_{0}}{T}
+
2C(R)C(\beta)^{\frac1p-\frac12}\epsilon
+
4\epsilon .
\end{align*}
It is sufficient to choose
\begin{align*}
T
\ge
\max\Bigg\{
&
2\left(1+C(R)C(\beta)^{\frac1p-\frac12}\right)
\left(C(R)G+\sigma\right)^{\frac{p}{p-1}}
\epsilon^{-\frac{p}{p-1}},
\\
&
\rho^{1/2}\Delta
\left(C(R)G+\sigma\right)^{\frac{p}{p-1}}
\epsilon^{-\left(\frac{p}{p-1}+\frac32\right)},
\\
&
\left(1+2(G+\sigma)\epsilon^{-1}\right)T_{0}
\Bigg\},
\end{align*}
where the last term can be further reformulated as $\mathcal{O}\left(\left(G+\sigma\right)^{\frac{p}{p-1} + 1}
\epsilon^{-\left(\frac{p}{p-1}+1\right)}\log\frac{G + \sigma}{\epsilon}\right)$ due to $T_0
:=
2+\frac{\log(1/(1-\beta_1))}
{\log(\sqrt{\beta_2}/\beta_1)} = \mathcal{O}\left(\left(\frac{G+\sigma}{\epsilon}\right)^{\frac{p}{p-1}}\log\frac{G + \sigma}{\epsilon}\right)$ when choosing $\beta_1
=
1-
\left(
\frac{\epsilon}{G+\sigma}
\right)^{\frac{p}{p-1}}$.

The third candidate ensures both \(T\ge T_0\) and that the burn-in term is
at most \(\epsilon\). The first candidate controls the first two
\(T^{-1}\) terms by at most \(2\epsilon\), and the second candidate controls
the \(\Delta\)-term by at most \(\epsilon\).  
Consequently,
\[
\mathbb{E}_{\tau}\Vert\nabla F(\widetilde{\rvx}_{\tau})\Vert^{[\rho]}
\le
\left(
8+
C(R)+
2C(R)C(\beta)^{\frac1p-\frac12}
\right)\epsilon .
\]

Meanwhile, we observe that the second term in the selection of T dominates w.r.t. $\epsilon$ under fixed constants and nondegenerate $\Delta, \rho$ asymptotically.
This proves the clipped stationarity guarantee for \(p\in(1,2]\).
\end{proof}

\end{document}